\documentclass{article}

\usepackage[preprint]{neurips_2024}

\usepackage[utf8]{inputenc} 
\usepackage[T1]{fontenc}    
\usepackage{hyperref}       
\usepackage{url}            
\usepackage{booktabs}       
\usepackage{amsfonts}       
\usepackage{nicefrac}       
\usepackage{microtype}      
\usepackage{xcolor}         
\usepackage{graphicx}
\usepackage{subfigure}
\usepackage{booktabs}
\usepackage{wrapfig}

\usepackage{amsmath}
\usepackage{amssymb}
\usepackage{mathtools}
\usepackage{amsthm}
\usepackage{bbm}
\usepackage[capitalize,noabbrev]{cleveref}

\newtheorem{theorem}{Theorem} 

\newtheorem{lemma}{Lemma}
\newtheorem{corollary}{Corollary}
\newtheorem{definition}{Definition}

\newtheorem{remark}{Remark}

\title{Contrasting Adversarial Perturbations: \\ The Space of  Harmless Perturbations}

\author{
\textbf{Lu Chen}$^{1}$\quad\textbf{Shaofeng Li}$^{2}$\quad\textbf{Benhao Huang}$^{1}$\quad\textbf{Fan Yang}$^{1}$\quad\textbf{Zheng Li}$^{1}$\quad\textbf{Jie Li}$^{1}$\quad\textbf{Yuan Luo}$^{1}$\\[2pt]
$^{1}$Shanghai Jiao Tong University \\ $^{2}$Peng Cheng Laboratory\\
\texttt{\{lu.chen,hbh001098hbh,fan-yang,li-zheng,lijiecs\}@sjtu.edu.cn}\\
\texttt{lishf@pcl.ac.cn}, \texttt{luoyuan@cs.sjtu.edu.cn}
}

\begin{document}

\maketitle

\begin{abstract}
Existing works have extensively studied adversarial examples, which are minimal perturbations that can mislead the output of deep neural networks (DNNs) while remaining imperceptible to humans. However, in this work, we reveal the existence of a harmless perturbation space, in which perturbations drawn from this space, regardless of their magnitudes, leave the network output unchanged when applied to inputs. Essentially, the harmless perturbation space emerges from the usage of non-injective functions (linear or non-linear layers) within DNNs, enabling multiple distinct inputs to be mapped to the same output. For linear layers with input dimensions exceeding output dimensions, any linear combination of the orthogonal bases of the nullspace of the parameter consistently yields no change in their output. For non-linear layers, the harmless perturbation space may expand, depending on the properties of the layers and input samples. 
Inspired by this property of DNNs, we solve for a family of general perturbation spaces that are redundant for the DNN's decision, and can be used to hide sensitive data and serve as a means of model identification. Our work highlights the distinctive robustness of DNNs (\textit{i.e.}, consistency under large magnitude perturbations) in contrast to adversarial examples (vulnerability for small imperceptible noises). 
\end{abstract}


\section{Introduction}
\label{sec:introduction}

The robustness of Deep Neural Networks (DNNs) against structured and unstructured perturbations has attracted significant attention in recent years~\citep{szegedy2013intriguing, nguyen2015deep, fawzi2016robustness, salman2021unadversarial}. 
In particular, deep learning models are shown highly vulnerable to adversarial perturbations~\citep{szegedy2013intriguing}. These well-crafted perturbations, which are imperceptibly small to the human eye, cause DNNs to misclassify with high confidence~\citep{carlini2017towards, madry2017towards, croce2020reliable}. 
Naturally, an inquiry arises: 

\textit{\quad Are there perturbations within the input space capable of preserving network output invariance?}


Unlike vulnerability against adversarial examples, in this paper, we reveal the robustness of DNNs to specific perturbations that render the network output mathematically strictly invariant. We demonstrate the existence of such \textit{harmless} perturbations that, when 
introduced onto natural images or embeddings, regardless of their magnitude, will not affect the discrimination of the DNN. 
Such harmless perturbations arising from the linear layers are universal, as they are instance-independent and solely determined by the parameter space of the DNN. These harmless perturbations span a continuous harmless subspace, embedded within the high-dimensional feature space. The surprising existence of harmless perturbations and their subspaces reveals a distinctive view of DNN robustness.


%

For the linear layers of DNNs, we find that when its input dimension $n$ exceeds the output dimension $m$, the harmless perturbation subspace of this layer can be derived by computing the \textit{nullspace} of its parameter matrix $A$, \textit{i.e.}, $N(A)=\{v\in\mathbb{R}^n| Av = \mathbf{0}\}$. To this end, the harmless subspace exhibits a dimension of $(n-m)$ and is embedded within an $n$-dimensional feature space.
Furthermore, the harmless perturbation space \textit{may} expand when involving non-linear layers, depending on the specific non-linear functions and input samples (See Section~\ref{sec:non_linear}). 
Inspired by the harmless subspace of linear layers, we further investigate the robustness of DNNs against more general perturbations, \textit{i.e.,} random noises or adversarial perturbations. We find that a family of those general perturbations, irrespective of their magnitude, identically influence the DNN's output. This phenomenon stems from the decomposition of arbitrary perturbations into the sum of any harmless and harmful components. Consequently, the network output for general perturbations becomes equivalent to that of harmful perturbations, particularly aligning with that of components orthogonal to the harmless perturbation subspace (Fig.~\ref{Fig:harmless_perturbations}(b)). Essentially, for any linear layer with a harmless subspace, the equivalent feature space is characterized by identical orthogonal components, leading to consistent network outputs. 

\begin{figure*}[t]
\begin{center}
\vskip -0.2in
\includegraphics[width=1.0\textwidth]{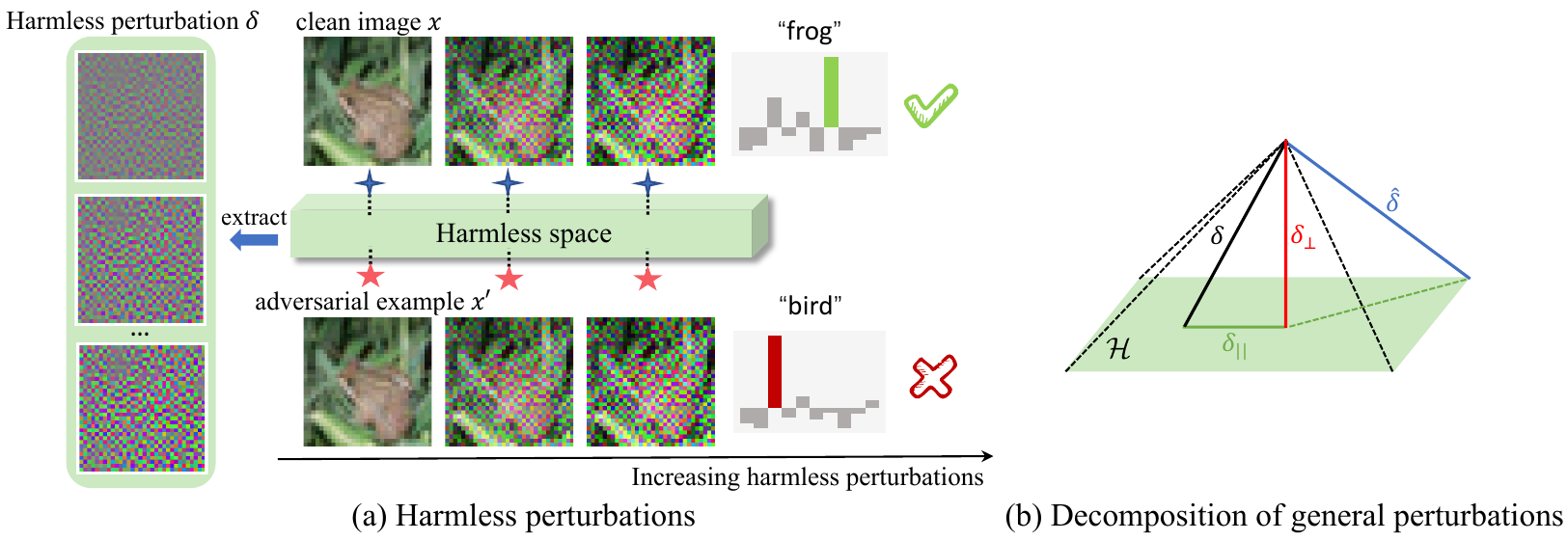}
\vskip -0.1in
\caption{(a) Harmless perturbations added to images completely do not change the network output of the images, regardless of the magnitude of these harmless perturbations. (b) Illustration of the equivalent effect of any perturbation on the network output. Given any linear layer with a harmless subspace $\mathcal{H}$, the network outputs of any perturbations $\delta$ and $\hat{\delta}$ are equivalent to those of their components $\delta_{\bot}$ orthogonal to the harmless subspace.}
\label{Fig:harmless_perturbations}
\end{center}
\vskip -0.2in
\end{figure*}



The existence of harmless perturbations and their space promotes several potential benefits. First, capitalizing on the disparity between DNNs and human perception, \textit{i.e.}, significant perturbations perceivable by the human eye may not affect the recognition of DNNs, we delve into the application of harmless perturbations to privacy-preserving data and model fingerprints. 
Additionally, as demonstrated in Fig.~\ref{Fig:harmless_perturbations}(a), there exist equivalent adversarial spaces, ensuring equal attacking capabilities for adversarial perturbations regardless of their magnitude. In other words, the perturbation magnitude is not a decisive factor in attacking the network. Instead, focusing on the attack utility of the ``effective component" of the perturbation facilitates a deeper understanding of the robustness of DNNs. In summary, this paper makes the following contributions:
\begin{itemize}
    \item We demonstrate for the first time the concept of ``harmless perturbations" and show the existence of a harmless perturbation space for DNNs. For any linear layer with the input dimension $n$ exceeding the output dimension $m$, there exists a continuous harmless perturbation subspace of dimension $(n-m)$. The harmless perturbation space \textit{may} expand when considering non-linear layers, depending on the properties of the layers and input samples. 
    \item We present a novel perspective to decompose any general perturbation (\textit{i.e.,} random noises or adversarial perturbations) into its harmful and harmless counterparts. Given any linear layer with a harmless perturbation subspace, the network output solely depends on its orthogonal (harmful) component, irrespective of its magnitude (innocuous) part.

    \item We reveal the difference between DNNs and human perception, \emph{i.e.}, significant perturbations captured by humans may not affect the recognition of DNNs, which highlights a distinctive aspect of DNN robustness. 
    Based on this insight, we employ the proposed harmless perturbations with a large magnitude to hide the sensitive image data for DNN usage. As harmless perturbations are usually not transferable across different DNNs, they can also serve as model fingerprints.  
\end{itemize}
\section{Related work}

\textbf{Adversarial examples and adversarial robustness.}
Existing literature extensively explored the impact of adversarial perturbations~\citep{szegedy2013intriguing} on the robustness of DNNs, including their ability to deceive both the digital and physical scenarios~\citep{kurakin2017adversarial}, fool both the white-box models~\citep{goodfellow2015explaining,madry2017towards} and black-box models~\citep{papernot2017practical, chen2017zoo}, and manifest as either image-specific or image-agnostic universal perturbations~\citep{moosavi2017unadversarial}. Many defenses against these adversarial perturbations have been proposed but they were susceptible to being broken by more powerful or adapted attacks~\citep{carlini2017towards, athalye2018obfuscated}. Amongst them, adversarial training~\citep{madry2017towards} and its variant~\citep{zhang2019theoretically} still indicated their relatively reliable robustness against more powerful attack~\citep{croce2020reliable}.

\noindent {\bf Adversarial space.} Previous studies have delved into the vulnerability of DNNs from the perspective of high-dimensional input spaces.~\citet{goodfellow2015explaining} argued that the ``highly linear" of DNNs explained their instability to adversarial perturbations.~\citet{fawzi2016robustness} quantified the robustness of classifiers from the dimensionality of subspaces within the semi-random noise regime.~\citet{justin2018spheres} suggested that adversarial perturbations arised from the high-dimensional geometry of data manifolds.~\citet{florian2017space} stated that adversarial transferability arised from the intersection of high-dimensional adversarial subspaces from different models.~\citep{Shafahi2020inevitable} empirically discussed that how dimensionality affected the robustness of classifiers to adversarial perturbations.

\textbf{Unrecognizable features.}
A series of prior works~\citep{geirhos2019imagenet, ilyas2019adversarial, tsipras2019robustness, jacobsen2019excessive, yin2019fourier, wang2020high} have demonstrated that humans and DNNs tend to utilize different features to make decisions. Besides, producing totally unrecognizable images~\citep{nguyen2015deep} or introducing visually perceptible patches to images~\citep{salman2021unadversarial, wang2022defensive, si2023angelic} may not alter the classification categories of DNNs.





\section{The space of harmless perturbations}
We develop a framework to rigorously define ``harmless'' and ``harmful'' perturbations \textit{w.r.t.} the network output. In particular, we formally define and solve for 
the subspace for harmless perturbations in any linear layer of a given DNN. Subsequently, the definitions and solutions are extended to non-linear layers by analyzing the properties of the functions. 

\subsection{Harmless perturbations for a linear layer} \label{sec:poc}
Consider a function mapping $\mathcal{L}: \mathbb{R}^n \mapsto \mathbb{R}^m$ on an input sample $x\in\mathbb{R}^n$, the goal is to find a set of input perturbations $\delta\in\mathbb{R}^n$ that rigorously do not change the output of the function. To this end, we first give the definition of harmless perturbations as follows.


\begin{definition}[Harmless perturbations] \label{def:set_of_harmless_perturbations}
The set of harmless perturbations for a function $\mathcal{L}$ is defined as $\mathcal{S} \coloneqq \{\delta | \mathcal{L}(x+\delta) = \mathcal{L}(x)\}$ subject to $\|\delta\|_{p}<\xi$, $\xi > 0$.
\end{definition}
\cref{def:set_of_harmless_perturbations} denotes a set of input perturbations that \textit{thoroughly} do not affect the function output. 
According to \cref{def:set_of_harmless_perturbations}, the set of harmless perturbations for a linear function $\mathcal{L}(x) = Ax$, where $A\in\mathbb{R}^{m\times n}$ is the parameter matrix, can be formulated as $\mathcal{S}=\{\delta|A(x+\delta)=Ax\}=\{\delta|A\delta=\mathbf{0}\}$. It indicates that the set of harmless perturbations for a single linear layer $\mathcal{L}$ is equivalent to the \textit{nullspace} of the parameter matrix $A$, \textit{i.e.}, $\mathcal{S} = N(A)=\{v\in\mathbb{R}^n| Av = \mathbf{0}\}$.

\begin{theorem}[Dimension of harmless perturbation subspace] \label{theorem1} 
Given a linear layer $\mathcal{L}(x) = Ax \in \mathbb{R}^{m}$ and an input sample $x\in\mathbb{R}^{n}$, where the parameter matrix $A\in\mathbb{R}^{m\times n}$. The dimension of the subspace for harmless perturbations is $dim(\mathcal{S}) = n-rank(A)$.
\end{theorem}  
Theorem~\ref{theorem1} demonstrate that the subspace for harmless perturbations is the span of $dim(\mathcal{S})$ linearly independent vectors $U\subset \mathcal{S}$, \textit{i.e.}, $\mathcal{S} = span(U)=\{\sum_{i=1}^{dim(\mathcal{S})}c_iu_i | c_i\in \mathbb{R}, u_i \in U\}$ (Proof is in~\cref{appx:theorem1}).
As a special case, the parameter matrix $A$ of a linear layer in DNNs learned through an optimization algorithm (\textit{e.g.}, SGD) starting from an arbitrary initialization, usually possesses linearly independent row vectors. So the dimension of the harmless perturbation subspace for a linear layer $\mathcal{L}(x) = Ax \in \mathbb{R}^{m}$ is $dim(\mathcal{S}) = n-m$.

\begin{remark}[Proof in~\cref{appx:remark2}]\label{remark2} Consider the case that the input dimension of the linear layer is less than or equal to the output dimension, \textit{i.e.}, $n\le m$. In this case, if the column vectors of the parameter matrix $A$ are linearly independent, then the dimension of the subspace for harmless perturbations is $dim(\mathcal{S}) =0$.
\end{remark}

\cref{remark2} state that there exists \textit{no} (non-zero) harmless perturbation that does not affect the output of the linear layer when $n\le m$ and $rank(A) = n$.

\subsection{The space of harmless perturbations for DNNs} \label{sec:non_linear}
Extending harmless perturbations from a single linear layer to the entire DNN is challenging. Consider a DNN $f: \mathbb{R}^{n_\text{in}} \mapsto \mathbb{R}^{n_\text{out}}$ on an input sample $x\in \mathbb{R}^{n_\text{in}}$, the goal now is to identify a set of harmless input perturbations $\delta\in\mathbb{R}^{n_\text{in}}$ which ultimately do not alter the network output. Notice that harmless perturbations solved for the intermediate layers do not influence subsequent layers. Therefore, we can formally define the set of harmless perturbations layer by layer for a given DNN.

\begin{definition}[Set of harmless perturbations for DNNs] \label{def:set_of_harmless_perturbations_features}
    The set of harmless perturbations on the $\!(l\!+\!1\!)$-th layer of a DNN $f$ is defined as $\mathcal{H}^{(l)} \coloneqq \{\delta^{(l)} | f^{(l+1)}(z^{(l)}+\delta^{(l)}) = f^{(l+1)}(z^{(l)})\}$.
\end{definition}
$z^{(l)}\in \mathbb{R}^{n^{(l)}}$ in~\cref{def:set_of_harmless_perturbations_features} represents the $l$-th intermediate-layer features of the input sample $x$, and $\delta^{(l)}$ denotes the perturbations added to the features $z^{(l)}$. \cref{def:set_of_harmless_perturbations_features} shows that if the set of harmless perturbations on the features can be found, these perturbations leave the network output unaffected. Furthermore, if we identify a set of perturbations on the input $\mathcal{P}^{(l)}\coloneqq \{\delta|  z^{(l)}+\delta^{(l)} = (f^{(l)}\circ\cdots\circ f^{(1)})(x+\delta), \forall \delta^{(l)}\in \mathcal{H}^{(l)} \}$ such that $\delta^{(l)} \in \mathcal{H}^{(l)}$, then $\mathcal{P}^{(l)}$ do not alter the network output.



\begin{lemma}[Proof in~\cref{appx:lemma1}] \label{lemma:space_of_harmless_perturbations}
    The set of harmless perturbations on the input for a DNN $f$ with $L$ layers is derived as $\mathcal{P} = \bigcup_{l=0}^{L-1} \mathcal{P}^{(l)}$
    , $\mathcal{P}^{(0)} \coloneqq \mathcal{H}^{(0)}$, $\mathcal{P}\subset \mathbb{R}^{n_\text{in}}$.
\end{lemma}
Lemma~\ref{lemma:space_of_harmless_perturbations} suggests that the set of harmless input perturbations for the entire DNN is the union of the corresponding set of harmless input perturbations $\mathcal{P}^{(l)}$ on each layer. Theoretically,~\cref{lemma:space_of_harmless_perturbations} does not restrict whether any layer in the DNN is linear or nonlinear, \textit{i.e.}, given any layer, if $\mathcal{H}^{(l)}$ and $\mathcal{P}^{(l)}$ can be evaluated, then harmless input perturbations for this layer can still be obtained.
Based on~\cref{lemma:space_of_harmless_perturbations}, we further investigate the effect of a single layer of nonlinearity on the harmless perturbation space. In scenarios involving non-linear layers, the harmless perturbation space \textit{may} expand, depending on the specific non-linear functions and input samples.

\setcounter{lemma}{0} 
\counterwithin*{lemma}{part}
\renewcommand{\thelemma}{1.\arabic{lemma}}

\begin{lemma}[Harmless perturbations for injective functions] \label{lemma:injective_function}
If the $\!(l\!+\!1\!)$-th layer $f^{(l+1)}$ is an injective function, the set of harmless perturbations on the $\!(l\!+\!1\!)$-th layer of a DNN $f$ is $\mathcal{H}^{(l)} = \{\mathbf{0}\}$. Otherwise, $\mathcal{H}^{(l)} \ne \{\mathbf{0}\}$. (Proof is in~\cref{appx:injective_function})
\end{lemma}

\begin{lemma}[Harmless perturbations for ReLU layers] \label{lemma:relu_layer}
    Suppose $f^{(l+1)}$ is the \text{\rm ReLU} layer, $\mathcal{H}^{(l)} = \{ \delta^{(l)}|\forall i, \delta^{(l)}_i =\begin{cases} 0,& z^{(l)}_i>0\\ t(\forall t\le-z^{(l)}_i),&  z^{(l)}_i\le0 \end{cases} $$\}$, which is determined by intermediate-layer features $z^{(l)}$ and hence the input sample $x$. (Proof is in~\cref{appx:relu_layer})
\end{lemma}

\begin{lemma}[Harmless perturbations for Softmax layers] \label{lemma:softmax_layer}
    Suppose $f^{(l+1)}$ is the \text{\rm Softmax} layer,  $\mathcal{H}^{(l)}=\{c\cdot\mathbf{1}, c\in\mathbb{R}\}$. (Proof is in~\cref{appx:softmax_layer})
\end{lemma}

\begin{lemma}[Harmless perturbations for Average Pooling layers] \label{lemma:average_pooling_layer}
    Suppose $f^{(l+1)}$ is the \text{\rm Average Pooling} layer, $\mathcal{H}^{(l)}=N(A_{\text{\rm avg}})$. $A_{\text{\rm avg}}$ is a coefficient matrix determined by the constraints that must be satisfied by the perturbations within each averaging region. (Proof is in~\cref{appx:average_pooling_layer})
\end{lemma}

\begin{lemma}[Harmless perturbations for Max Pooling layers] \label{lemma:max_pooling_layer}
    Suppose $f^{(l+1)}$ is the \text{\rm Max Pooling} layer, $\mathcal{H}^{(l)}=\{\forall p,i, \delta^{(l)}_{p,i} \le c_p-z^{(l)}_{p,i}\} \cap \{\forall p, \prod_{j=1}^{k\times k} (\delta^{(l)}_{p,j}-c_p+z^{(l)}_{p,j}) = 0\}$. $c_p \coloneqq  \textit{\rm max} \{z^{(l)}_{p,1}, z^{(l)}_{p,2}, \cdots, z^{(l)}_{p,k\times k} \}$ is the maximum value of features within the $k\times k$ region of the $p$-th patch. $\mathcal{H}^{(l)}$ is determined by intermediate-layer features $z^{(l)}$ and hence the input sample $x$. (Proof is in~\cref{appx:max_pooling_layer})
\end{lemma}

\begin{theorem}[Harmless perturbations for two-layer neural networks] \label{theorem_non_linear} 
Given a two-layer neural network $f(x) = \sigma(Ax)$, where $\sigma$ represents any function. If $\sigma$ is an injective function, the set of harmless perturbations on the input $\mathcal{P}$ for $f$ is $\mathcal{P} = \mathcal{P}^{(0)}$. Otherwise, $\mathcal{P} = \mathcal{P}^{(0)}\cup \mathcal{P}^{(1)}\supseteq \mathcal{\mathcal{P}}^{(0)}$. Here, $\mathcal{P}^{(1)}=\{\delta|A\delta = \delta^{(1)}, \forall \delta^{(1)}\in\mathcal{H}^{(1)}\cap C(A)\}$\footnote{Note that the equation $A\delta = \delta^{(1)} (\delta \ne \mathbf{0})$ has a solution (meaning at least one solution) if and only if $\delta^{(1)}$ is in the column space of $A$, \textit{i.e.}, $\delta^{(1)}\in C(A)$.} is determined by the specific function $\sigma$ and the input sample $x$. (Proof is in \cref{appx:theorem_non_linear}) 
\end{theorem} 
\cref{theorem_non_linear} suggests that the property of the function $\sigma$ determines whether the set of harmless perturbations for $Ax$ may expand. For instance, if $\sigma$ is an injective function, such as Sigmoid, Tanh, leaky ReLU~\citep{maas2013rectifier}, exponential linear
unit (ELU)~\citep{clevert2016fast} and scaled exponential linear unit (SeLU)~\citep{klambauer2017self-normalizing} activation functions, and the linear Batch Normalization (BN) layers at inference time~\citep{Ioffe2015BN}, the set of harmless perturbations on the input $\mathcal{P}$ 
remains unchanged, compared to that of $Ax$. Conversely, if $\sigma$ is a non-injective function, such as ReLU~\citep{nair2010ReLU}, Softmax, Average Pooling~\citep{Yann1990AveragePooling}, and Max Pooling layers~\citep{Scherer2010Evaluation} (see~\cref{lemma:relu_layer,lemma:softmax_layer,lemma:average_pooling_layer,lemma:max_pooling_layer} and~\cref{theorem_non_linear} for their $\mathcal{P}$, respectively), the set of harmless perturbations on the input $\mathcal{P}$ \textit{may} expand $\mathcal{P}\supseteq \mathcal{P}^{(0)}$, depending on the specific functions and input samples. Note that, in the above non-linear layers, the harmless perturbation space for the ReLU layer is determined by the input sample $x$. In an extreme case, if every element of $Ax$ is positive, then its harmless perturbation subspace $\mathcal{P} = \mathcal{P}^{(0)}$. Otherwise, if every element of $Ax$ is not positive, $\mathcal{P} = \mathcal{P}^{(0)}\cup \mathcal{P}^{(1)}\supseteq \mathcal{\mathcal{P}}^{(0)}$. (For more details, please refer to~\cref{lemma:relu_layer} in~\cref{appx:relu_layer}). 
In summary, the harmless perturbation space on the input \textit{does} expand $\mathcal{P}\supsetneq \mathcal{P}^{(0)}$ if there exists at least one harmless perturbation $\delta^{(1)}\in \mathcal{H}^{(1)}\cap C(A) (\delta^{(1)}\ne \mathbf{0})$ for the non-injective function $\sigma$. 

\begin{lemma}[Harmless perturbations for two-layer linear networks] \label{lemma:two_layer_neural_networks}
    Given a two-layer linear network $f(x) = A_2A_1x$,  $\mathcal{P} = \mathcal{P}^{(0)}\cup \mathcal{P}^{(1)}\supseteq \mathcal{P}^{(0)} $. Here, $\mathcal{P}^{(0)}=N(A_1)$ and $\mathcal{P}^{(1)}=\{\delta|A_1\delta = \delta^{(1)},\forall \delta^{(1)}\in N(A_2)\cap C(A_1)\}$. (Proof is in~\cref{appx:average_pooling_layer})
\end{lemma}

Furthermore, \cref{lemma:two_layer_neural_networks} illustrates 
the expansion of harmless perturbations on the input $\mathcal{P}$ solely depends on the dimensions of those two linear layers. For two common scenarios in DNNs, where given $A_1\in\mathbb{R}^{d \times n}$ and $A_2\in\mathbb{R}^{m \times d}$, when $n,m>d$, $\mathcal{P} = \mathcal{P}^{(0)}$. Otherwise, when $n,m<d$, $\mathcal{P} = \mathcal{P}^{(1)}$. (Please see~\cref{lemma:two_layer_neural_networks} in~\cref{appx:average_pooling_layer} for the details.)

\begin{figure*}[t]
\centering
\vskip -0.2in
\subfigure[Convolutional layer]{\includegraphics[width=0.49\textwidth]{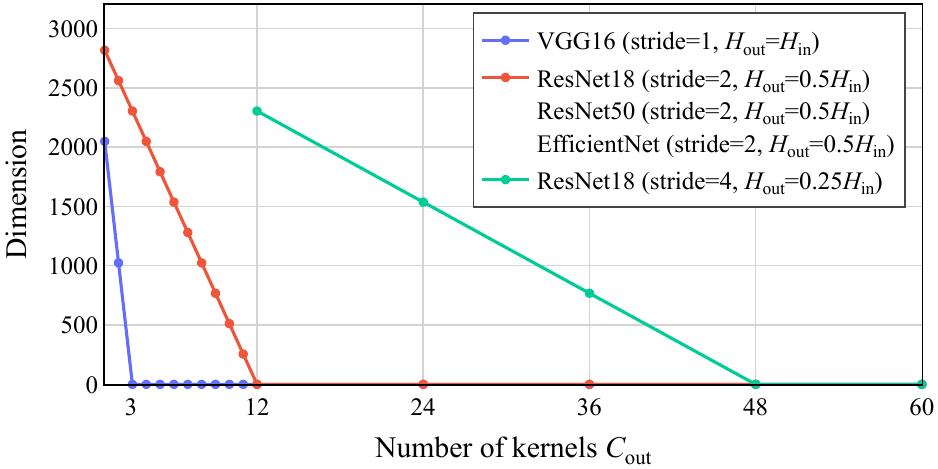}\label{fig:corollaries_harmless_perturbation_conv}}
\subfigure[Fully-connected layer]{\includegraphics[width=0.49\textwidth]{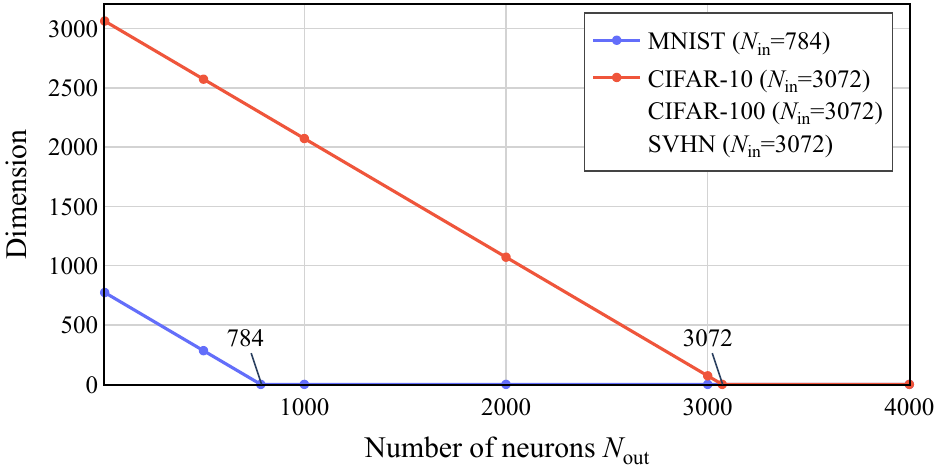}\label{fig:corollaries_harmless_perturbation_fc}}
\vskip -0.15in
\caption{Dimension of harmless perturbation subspace for (a) convolutional layers and (b) fully-connected layers. When the input dimension $n$ of the linear layer is larger than the output dimension $m$, the dimension of the harmless subspace is $(n-m)$. Otherwise, the dimension is 0. 
}

\vskip -0.25in
\label{Fig:corollaries_harmless_perturbation}
\end{figure*}

\subsection{The subspace of harmless perturbations for linear layers in DNNs}
\label{subsec:subspace_harmless_perturbations_for_linear_layers}

Nevertheless, in this section, we focus on the set of harmless perturbations for two classical linear layers in DNNs, \textit{i.e.}, convolutional layers and fully-connected layers. 

\begin{corollary}[Harmless perturbation subspace for convolutional layers, proof in~\cref{appx:corollary1}] \label{corollary:convolutional_layer} Given a convolutional layer $f^{(l+1)}$ with linearly independent vectorized kernels whose kernel size is not smaller than the stride,  $z^{(l+1)}=f^{(l+1)}(z^{(l)})\in\mathbb{R}^{C_{\text{\rm out}} \times H_{\text{\rm out}} \times W_{\text{\rm out}}}$ and $z^{(l)}\in \mathbb{R}^{C_{\text{\rm in}} \times H_{\text{\rm in}} \times W_{\text{\rm in}}}$. If the input dimension is greater than the output dimension, then the dimension of the subspace for harmless perturbations is $dim(\mathcal{H}^{(l)}) = C_{\text{\rm in}}  H_{\text{\rm in}}  W_{\text{\rm in}} - C_{\text{\rm out}}  H_{\text{\rm out}} W_{\text{\rm out}}$. Otherwise, $\mathcal{H}^{(l)} = \{\mathbf{0}\}$.
\end{corollary}
Corollary~\ref{corollary:convolutional_layer} demonstrates the subspace for harmless perturbations in a convolutional layer is the span of $dim(\mathcal{H}^{(l)})$ linearly independent vectors $U\subset \mathcal{H}^{(l)}$. 
Specifically, $\mathcal{H}^{(l)}$ can be obtained by computing the nullspace of a matrix $A\in \mathbb{R}^{(C_{\text{\rm out}}  H_{\text{\rm out}} W_{\text{\rm out}})\times  (C_{\text{\rm in}}  H_{\text{\rm in}}  W_{\text{\rm in}})}$. 
In practice, $A$ is affected by the padding and the stride of the convolutional layer (see~\cref{appx:equivalent_matrix} for details). 
Similarly, given a fully-connected layer $z^{(l+1)}= W^\top z^{(l)}\in \mathbb{R}^{N_{\text{\rm out}}}$ and $z^{(l)}\in \mathbb{R}^{N_{\text{\rm in}}}$, 
the harmless subspace is the span of $dim(\mathcal{H}^{(l)}) = N_{\text{\rm in}} - N_{\text{\rm out}}$ linearly independent vectors $U\subset \mathcal{H}^{(l)}$ (see~\cref{appx:corollary2} in~\cref{appx:corollary1}). Here, $\mathcal{H}^{(l)}$ is computed as the nullspace of a matrix $A=W^\top$.

Experiments on various DNNs verify Corollaries~\ref{corollary:convolutional_layer} and~\ref{appx:corollary2}. 
In~\cref{Fig:corollaries_harmless_perturbation}, the dimension of the harmless perturbation subspace $dim(\mathcal{H}^{(l)})$ decreased as the output dimension increased.
When the output dimension exceeds the input dimension, $dim(\mathcal{H}^{(l)})$ becomes 0. 
Specifically, we verified the dimension of the harmless perturbation subspace for convolutional layers using various DNNs, including ResNet-18/50~\citep{he2016deep}, VGG-16~\citep{simonyan2014very} and EfficientNet~\citep{tan2019efficientnet}, on the CIFAR-10 dataset~\citep{krizhevsky2009learning}. Here, we modified the feature size of the output of the first convolutional layer by setting different strides (see~\cref{appx:verifying_dimension}). Furthermore, we verified the dimension of the harmless perturbation subspace for fully-connected layers using the MLP-5 on various datasets, including the MNIST dataset~\citep{LeCun:2010}, the CIFAR-10/100 dataset~\citep{krizhevsky2009learning} and the SHVN dataset~\citep{netzer2011reading}, to compare the dimension of the subspace under different input dimensions.

Conversely, there exists \textit{no} (non-zero) perturbation making the network output invariant, if the input dimension of a given linear layer is not greater than the output dimension. 
However, the least harmful perturbation can be solved for such that the layer output is minimally affected, \textit{i.e.}, given the matrix $A$ with equivalent effect of a linear layer,  the least harmful perturbation is $(\delta^{(l)})^\ast = {\rm arg min}_{\delta^{(l)}}\| A \delta^{(l)}\|_2$, \textit{s.t.,} $\|\delta^{(l)}\|_2=1$. Hence, the least harmful perturbation $(\delta^{(l)})^\ast$ is the eigenvector corresponding to the smallest eigenvalue of the matrix $A^\top A$ (see~\cref{lemma:least_harmful} in~\cref{appx:average_pooling_layer}).

We validated the impact of harmless perturbations and the least harmful perturbations on network performance across varying perturbation magnitudes. 
In~\cref{fig:corollaries_harmless_perturbation_conv}, 
harmless perturbations, regardless of their magnitude, do not affect the discrimination of DNNs (see~\cref{appx:evaluating_performance} for details). For the least harmful perturbations in~\cref{fig:corollaries_harmless_perturbation_fc}, 
they also have negligible effects on the network
performance, compared with the Gaussian noise $\mathcal{N}(0,1)$ added to each pixel. 
Furthermore, we evaluated the root mean squared error (RMSE) between the network outputs of the perturbed images $\hat{y}_x$ and the network outputs of natural images $y_x$ on the ResNet-50, \textit{i.e.},  RMSE=$\mathbb{E}_x [\frac{1}{\sqrt{n}}\vert\vert \hat{y}_x - y_x\vert\vert]$.~\cref{tab:rmse_of_perturbations} further demonstrates that compared to adversarial perturbations and Gaussian noise, harmless perturbations completely did not change the network output with negligible
errors, and the least harmful perturbation had a weak impact on the network output as the perturbation magnitude increased.


\begin{figure*}[t]
\vskip -0.2in
\centering
\subfigure[Harmless perturbation]{\includegraphics[width=0.495\textwidth]{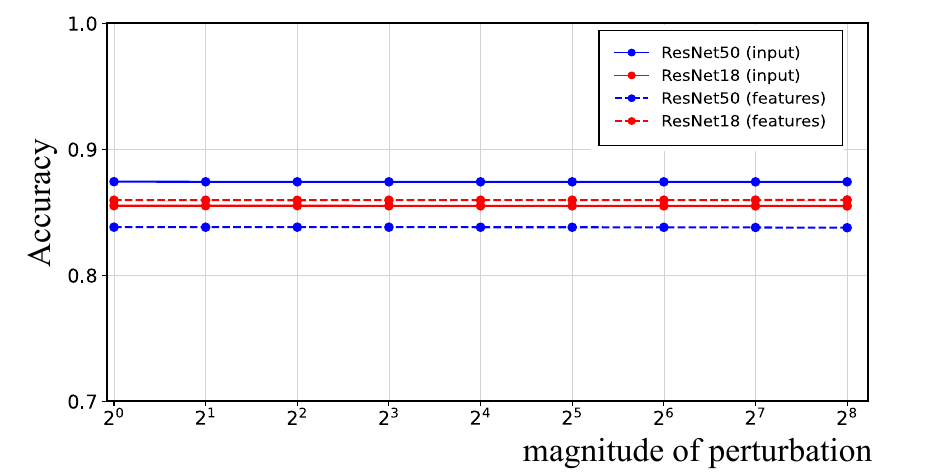}\label{fig:corollaries_harmless_perturbation_conv}}
\subfigure[Least harmful perturbation]{\includegraphics[width=0.495\textwidth]{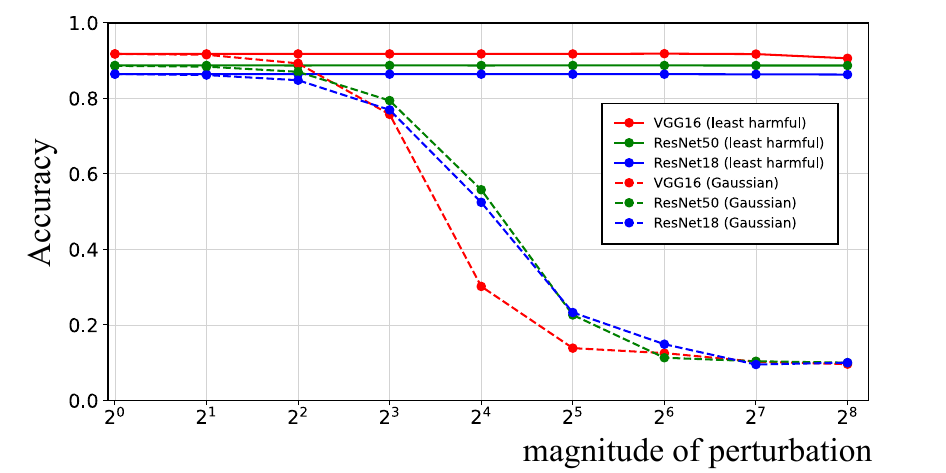}\label{fig:corollaries_harmless_perturbation_fc}}
\vskip -0.1in
\caption{The effect of perturbation magnitude on the performance of the network. 
We trained the CIFAR-10 dataset on various networks and tested the effect of varying magnitudes on (a) harmless perturbations and (b) the least harmful perturbations.}
\vskip -0.1in
\label{Fig:accuracy_of_harmless_perturbation}
\end{figure*}

\begin{table*}[t]
\vskip -0.1in
\caption{Root mean squared errors between the network outputs of the perturbed images and original images on the CIFAR-10 dataset.}
\label{tab:rmse_of_perturbations}
\begin{center}
\begin{small}
\resizebox{1.0\linewidth}{!}{
\begin{tabular}{lcccccc}
\toprule
 & $\epsilon$ & $2\epsilon$ & $4\epsilon$ & $8\epsilon$ & $16\epsilon$ & $32\epsilon$\\
\midrule
Gaussian noise & 0.1226 & 0.3154 & 0.8112 & 1.7871 & 3.3921 & 5.0009 \\
Adversarial perturbation~\citep{madry2017towards} & 6.1994 & 6.3225 & 5.5410 & 5.6122 & 6.6585 & 11.2747  \\
\midrule
Harmless perturbation  &\textbf{3.63e-15} & \textbf{3.70e-15} & \textbf{3.77e-15} & \textbf{4.20e-15} & \textbf{5.38e-15} & \textbf{8.55e-15} \\
Least harmful perturbation  &0.0003 & 0.0007 & 0.0013 & 0.0027 & 0.0053 & 0.0105 \\
\bottomrule
\end{tabular}}
\end{small}
\end{center}
\vskip -0.3in
\end{table*}

\section{Projection onto the harmless subspace}
\label{sec:projection_onto_subspace}

Inspired by the harmless subspace of linear layers, we can decompose any given perturbation (\textit{i.e.,} random noise, adversarial examples) into its two orthogonal counterparts, namely, harmful and harmless components. 
This section extends the harmless subspace to any given perturbations and investigates the projections of these perturbations onto their corresponding harmless subspaces.
\begin{theorem}[Arbitrary decomposition of perturbations, proof in~\cref{appx:theorem2}] \label{theorem2} Given the $(l+1)$-th linear layer with harmless subspace $\mathcal{H}^{(l)}\ne \{\mathbf{0}\}$ and any perturbation $\forall \delta^{(l)}\notin \mathcal{H}^{(l)}$, it can be arbitrarily decomposed into the sum of a harmless perturbation and a harmful perturbation, \textit{i.e.}, $\delta^{(l)} = \delta^{(l)}_a + \delta^{(l)}_b, \forall \delta^{(l)}_a \in \mathcal{H}^{(l)}$ and $\delta^{(l)}_b \notin \mathcal{H}^{(l)}$. Then, $ f^{(l+1)}(\delta^{(l)})= f^{(l+1)}(\delta^{(l)}_b)$. 
\end{theorem}  

\vskip -0.1in

Theorem~\ref{theorem2} indicates that 
the network output of any perturbation $\delta^{(l)}\notin \mathcal{H}^{(l)}$  is equivalent to that of its corresponding harmful component $\delta^{(l)}_b\coloneqq (\delta^{(l)} - \delta^{(l)}_a)\notin \mathcal{H}^{(l)}, \forall \delta^{(l)}_a \in \mathcal{H}^{(l)}$, \textit{no matter how large the $\ell_p$ norm of harmful component is}.\setcounter{lemma}{2} 
\renewcommand{\thelemma}{\arabic{lemma}}According to~\cref{theorem2}, an infinite number of perturbations, regardless of their magnitude, will induce the equivalence of a continuous harmful space\footnote{Note that the harmful space is not a linear subspace of $\mathbb{R}^{n^{(l)}}$, since it does not contain $\mathbf{0} \in \mathbb{R}^{n^{(l)}}$.}.
Naturally, an inquiry arises: what is the extent of these perturbations concerning a given DNN? Theorem~\ref{theorem3} extends the argument by establishing the existence of a \textit{unique} perturbation characterized by the smallest $\ell_2$ norm (see the proof in~\cref{appx:orthogonal_decomposition}). This perturbation is orthogonal to the harmless subspace, and exhibits network output consistent with the above infinite number of perturbations embedded in the continuous harmful space (Figure~\ref{Fig:harmless_perturbations}(b)).

\begin{theorem}[Orthogonal decomposition of perturbations, proof in~\cref{appx:theorem3}] \label{theorem3} Given the $(l+1)$-th linear layer with harmless subspace and any perturbation $\forall \delta^{(l)}\notin \mathcal{H}^{(l)}$, it has a unique decomposition $\delta^{(l)} = \delta^{(l)}_{\parallel} + \delta^{(l)}_{\bot}$with the parallel component $\delta^{(l)}_{\parallel} = P\delta^{(l)} \in \mathcal{H}^{(l)} $ and the  orthogonal component $\delta^{(l)}_{\bot} = (I-P)\delta^{(l)} \notin \mathcal{H}^{(l)}$. Then, $f^{(l+1)}(\delta^{(l)}_{\parallel})= \mathbf{0}$ and $f^{(l+1)}(\delta^{(l)}) = f^{(l+1)}(\delta^{(l)}_{\bot})$.
\end{theorem} 
\vskip -0.05in
$P= U(U^\top U)^{-1}U^\top$ represents the projection matrix onto the harmless subspace $\mathcal{H}^{(l)}\subset\mathbb{R}^{n^{(l)}}$, and $U\in\mathbb{R}^{n^{(l)}\times dim(\mathcal{H}^{(l)})}$denotes a set of $dim(\mathcal{H}^{(l)})$ orthogonal bases for the subspace $\mathcal{H}^{(l)}$. 

As a special case of~\cref{theorem2},~\cref{theorem3} demonstrates that the network output of a family of features/perturbations is equivalent to that of the component of this perturbation family, which is orthogonal to the subspace.
In essence, as expounded in~\cref{theorem4}, a collection of perturbations can be categorized as a perturbation family with identical impact on the network output, if their orthogonal components exhibit congruence in both magnitude and direction.


\begin{theorem}[Identical impact of a family of perturbations, proof in~\cref{appx:theorem4}] \label{theorem4} Given the $(l+1)$-th linear layer with harmless subspace and two different perturbations $\forall \delta^{(l)}\ne \hat{\delta}^{(l)}$ and $\delta^{(l)}, \hat{\delta}^{(l)} \notin \mathcal{H}^{(l)}$, if their orthogonal components are the same, \textit{i.e.}, $\delta_{\bot}^{(l)} = \hat{\delta}_{\bot}^{(l)}$, then $f^{(l+1)}(\delta^{(l)}) =  f^{(l+1)}(\hat{\delta}^{(l)})$.
\end{theorem}

Theorem~\ref{theorem4} posits that \textit{when a set of features/perturbations lies equidistant to the harmless subspace and exhibits the same direction in their orthogonal components}, these perturbations form a family that induces uniform network effects. These perturbations can be analogized to form contour lines in a topographic map, as these perturbations with the same orthogonal components yield the same network output (\cref{Fig:example_of_theorems}). Notably, this effect remains consistent irrespective of the perturbation magnitude. Furthermore, when the orthogonal components of any two perturbations have different directions, \textit{i.e.}, $\delta_{\bot}^{(l)} \ne \alpha \cdot \hat{\delta}_{\bot}^{(l)} (\alpha \in \mathbb{R})$, then the layer outputs are inconsistent $f^{(l+1)}(\delta^{(l)}) \ne  f^{(l+1)}(\hat{\delta}^{(l)})$ (\cref{lemma:different_direction_perturbation} in~\cref{appx:relu_layer}). It is also implied that orthogonal components with the same magnitude but different directions do not necessarily corrupt the network to the same output.
We believe that the perturbation decomposition approach presented in this work allows us to re-examine the intriguing properties of adversarial examples by decomposing the perturbations into their harmful and harmless counterparts.

\section{Applications of harmless perturbations}
\label{sec:applications}
\subsection{Privacy protection}
\label{privacy_protection}

We first consider a scenario where users may require employing a pre-trained model on a third-party server to analyze data containing sensitive information (\textit{e.g.,} facial, medical, and credit data)~\citep{schick2023Toolformer, shen2023HuggingGPT,wu2023VisualChatGPT,liang2023TaskMatrix}. Specifically, either the third-party server or the user provides a pre-trained model, enabling the user to access the network parameters. Subsequently, the user locally generates privacy-preserving data using the available parameters, and then deploys the protected data, along with the network, to the third-party server. 
To alleviate information leakage from sensitive data, harmless perturbations with sufficiently large magnitudes can be incorporated to original samples. This process renders the generated samples unrecognizable to humans, effectively obscuring sensitive information within the images, without compromising network performance.

To be specific, our goal is to generate a visually unrecognizable image, denoted as $\hat{x}\in \mathbb{R}^{n_{\text{in}}}$, to substitute the original image $x$, ensuring that its network output is identical with that of the original image $x$. 
Specifically, given a DNN with a harmless perturbation subspace  $\mathcal{H}^{(0)}\subset\mathbb{R}^{n_{\text{in}}}$ in its first linear layer, and a set of orthonormal bases $\{u_1, u_2,\cdots, u_{d}\}$ of the subspace $\mathcal{H}^{(0)}$ ($d=dim(\mathcal{H}^{(0)})$), visually unrecognizable harmless perturbations can simply be generated by maximizing the dissimilarity between the original image $x$ and the generated image $\hat{x}\coloneqq x + \sum_{i=1}^{d}c_iu_i, c_i\in\mathbb{R}, u_i\in\mathbb{R}^{n_{\text{in}}}$. Without loss of generality,
we quantify the difference between the two images using the Mean Squared Error (MSE), \textit{i.e.}, $\max_{\{c_1, c_2,\cdots, c_{d}\}} \frac{1}{n_{\text{in}}}\Vert \hat{x} - x \Vert_2^2$. 
To make the pixels of the generated image in the range $[0,1]$, we add two penalties on the pixels out of bounds, \textit{i.e.}, $\sum_i|\mathbbm{1}(\hat{x}_i<0)\cdot \hat{x}_i| + |\mathbbm{1}(\hat{x}_i>1)\cdot \hat{x}_i|$, as shown in Figure~\ref{Fig:harmless_perturbations}(a). 

\textbf{Recovering original images}. Reconstructing the original image $x$ from the generated image $\hat{x}$ is a challenging task even if the attacker can access network parameters.
Since the parameter matrix $A$ uniquely determine the harmless perturbation subspace, it is equivalent to specifying the subspace $\mathcal{H}^{(0)}$. However, according to~\cref{theorem2}, the generated image $\hat{x} \notin \mathcal{H}^{(0)}$ can be decomposed into the sum of an \textit{infinite} number harmless components $\hat{\delta} \in \mathcal{H}^{(0)}$ (Figure~\ref{Fig:harmless_perturbations}(b)) and reconstructed images $x^{\text{recon}} \coloneqq \hat{x} - \hat{\delta}, \forall \hat{\delta} \in \mathcal{H}^{(0)}$. Therefore, the original image cannot be uniquely determined when the magnitude and direction of the harmless perturbation are unknown.

\begin{wraptable}{r}{6.5cm}
\vskip -0.1in
\caption{Perceptual similarity between the perturbed images and the original images on the CIFAR-10 dataset.}
\label{tab:ssim_lpips}
\begin{center}
\resizebox{0.45\columnwidth}{!}{
\begin{tabular}{lccc}
\toprule
 & Harmless perturbation  & Gaussian noise\\
\midrule
SSIM ($\downarrow$) & \textbf{0.4719} & 0.6825 \\
LPIPS ($\uparrow$)& 0.2031 & \textbf{0.3007}  \\
$\Delta$accuracy ($\downarrow$) & \textbf{0.00\%} & 32.46\% \\
\bottomrule
\end{tabular}}
\end{center}
\vskip -0.1in
\end{wraptable}

\textbf{Visual indistinguishability.} To quantify the capability of the generated images in preserving privacy for human perception, we evaluated the perceptual similarity using two similarity metrics, \textit{i.e.},  the Structural Similarity Index (SSIM)~\citep{wang2004ssim} and the Learned Perceptual Image Patch Similarity (LPIPS)~\citep{zhang2018lpips} metrics. Besides, we evaluated the degradation in classification performance of the generated images, compared to the original images. We compared the privacy-preserving capability of the generated harmless perturbations with the Gaussian noise $\mathcal{N}(0,0.1^2)$ added to each pixel on the ResNet-50.~\cref{tab:ssim_lpips} shows that the generated harmless perturbations achieved a similar level of privacy preservation as Gaussian noise, but the harmless perturbations completely did not change the DNN's discrimination of images.


\subsection{Model fingerprint}
\begin{wrapfigure}{r}{6.5cm}
\centering
\vskip -0.1in
\includegraphics[width=1.0\linewidth]{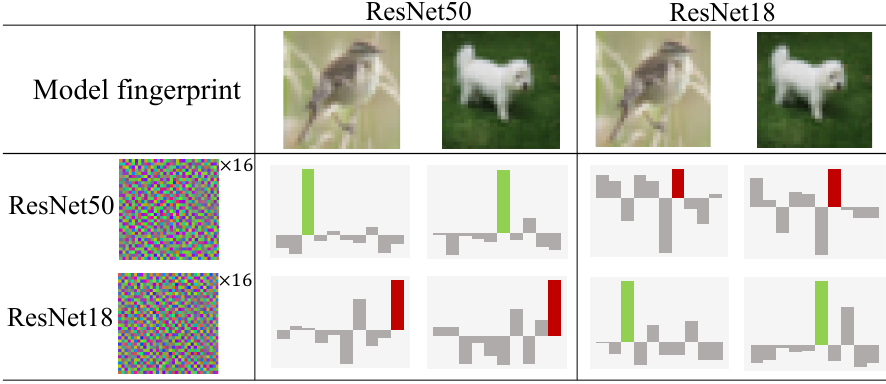}
\vskip -0.1in
\caption{Harmless perturbations (magnified by 16x) can serve as identity fingerprint for models, allowing for tracking changes in closed-source models. }
\label{Fig:model_fingerprint}
\vskip -0.1in
\end{wrapfigure}
Harmless perturbations also can be used for model fingerprints~\citep {finlayson2024logits,zeng2024human} to faithfully reflect the model's changes, as they are determined by the parameter space of the DNN. 
We demonstrate this usage by considering the simple case of establishing identity fingerprints for two DNNs. ~\cref{Fig:model_fingerprint} illustrates the network's response when adding two different harmless perturbations extracted from two distinct models on input images. 
Incorporating significant harmless perturbations generated by one model into various input samples preserves the outputs of that model, while applying them to input samples of another model leads to significant changes in the outputs of another model. This demonstrates harmless perturbations can potentially serve as model fingerprints.

{\bf Transferability of harmless perturbations.} Typically, given two DNNs with different parameters, their harmless perturbation spaces are not equal, \textit{i.e.}, $\mathcal{P}_1\ne \mathcal{P}_2$. However, there may exist few harmless perturbations that are transferable and serve as harmless perturbations for both DNNs, \textit{i.e.}, $\delta \in \mathcal{P}_1  \cap \mathcal{P}_2$. 
To avoid choosing those rare transferable harmless perturbations as model fingerprints, we constraint the sampling of model fingerprints solely from 
the non-intersecting harmless perturbation spaces of the two DNNs, satisfying $\delta_1 \in \mathcal{P}_1 - (\mathcal{P}_1 \cap \mathcal{P}_2)$ and $\delta_2 \in \mathcal{P}_2 - (\mathcal{P}_1 \cap \mathcal{P}_2)$.


\subsection{The intriguing properties of DNNs from harmless perturbations}
\label{subsec:intriguing_properties}


\textbf{Seeing is not always believing.} Surprisingly, we find that \textit{distances within the feature space may exhibit considerable variation between DNNs and human perception}. Human perceptual systems tend to discern non-equivalence when the magnitude of the perturbation added to a feature significantly exceeds the feature's magnitude. In contrast, \textit{DNNs tend to disregard the magnitude of features/perturbations}. 
DNNs are completely unaffected by such harmless perturbations, highlighting a distinctive aspect of DNN robustness. Furthermore, harmless perturbations invalidate distance-based similarity metrics, such as the widely used Euclidean and cosine distances~\citep{Mensink2013distance-based,zhang2018lpips}. 
For example, two vectors sampled from harmless/equivalent space, regardless of their magnitude or direction, may deemed dissimilar through these similarity metrics, yet deep networks still regard them as identical. Consequently, there arises a necessity to reassess whether these similarity metrics faithfully reflect the true modelling of similarity by deep networks.

\begin{figure}[t!]
\centering
\includegraphics[width=1.0\linewidth]{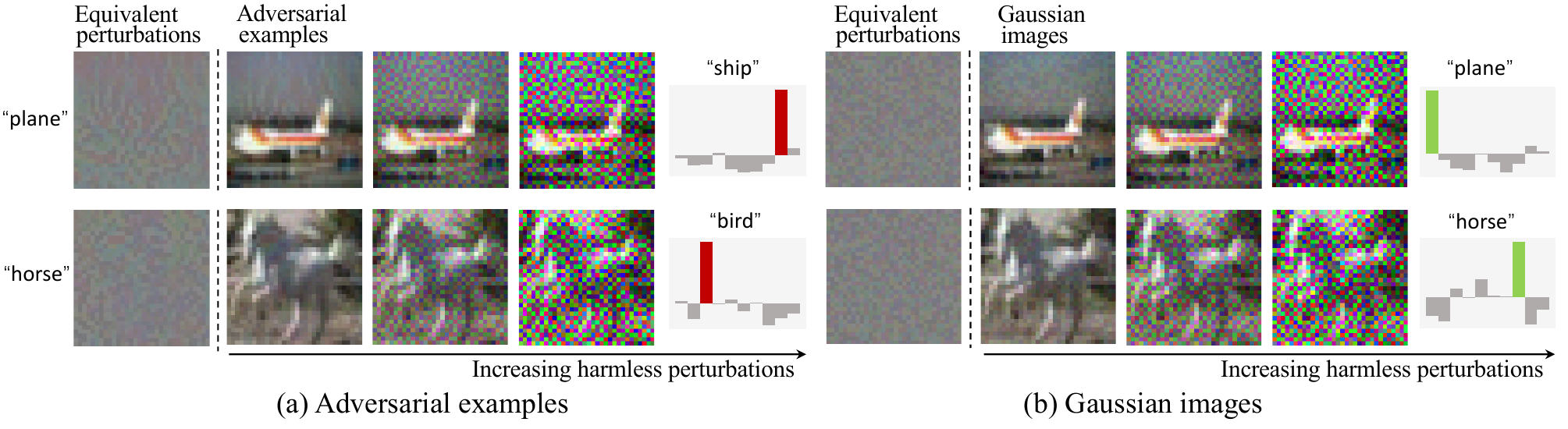}
\vskip -0.1in
\caption{Adding Harmless perturbations to images with different noises ((a) adversarial perturbations (b) Gaussian noises) completely do not change the network output of the perturbed images, regardless of the magnitude of these harmless perturbations. Perturbed images incorporating harmless perturbations of arbitrary magnitude, drawn from the equivalent perturbation space, exhibit an effect on the network output that is equivalent to the impact of images with equivalent (orthogonal) perturbations.}
\label{Fig:adversarial_gaussian}
\vskip -0.2in
\end{figure}


\textbf{Equivalent adversarial spaces.} 
We revisit the impact of perturbation range/magnitude~\citep{madry2017towards, wang2019convergence} of adversarial examples on the DNNs. ~\cref{theorem2,theorem3} demonstrate that \textit{infinitely large, infinitely numerous} features/perturbations are equivalent to their components orthogonal to the harmless subspace. 
Therefore, for any perturbation, \textit{there exist equivalent (adversarial) perturbation spaces}, ensuring equal attacking capabilities for perturbations.
Compared to the Gaussian noises in~\cref{Fig:adversarial_gaussian}(b), the adversarial perturbations which have similar perturbation magnitudes in~\cref{Fig:adversarial_gaussian}(a), lead to completely different attack utilities.  
Interestingly, all adversarial perturbations for each sample in~\cref{Fig:adversarial_gaussian}(a) have the same attack utility, irrespective of their magnitudes. 
The equivalent adversarial spaces imply that: 1) \textit{the perturbation magnitude is not a decisive factor attacking the network}, and 2) attention should be paid to the ``effective components" of perturbations, 
\textit{i.e.,} we can further decompose the perturbation in a more fine-grained way.
We believe that further exploration of the equivalent space helps to understand the robustness of DNNs.






\section{Conclusion}\label{sec:conclusion}
In this paper, we show the existence of harmless perturbations and their subspaces. Such harmless perturbations, regardless of their magnitude, render the network output completely unaltered. Essentially, the harmless perturbation space arises from the usage of non-injective functions within DNNs. We prove that for any linear layer in a DNN where the input dimension $n$ exceeds the output dimension $m$, there exists a continuous harmless subspace with a dimension of $(n-m)$. We further show the existence of the feature/perturbation space characterized by identical orthogonal components, consistently influencing the network output. Besides, the harmless perturbation space may expand when involving non-linear layers. Our work reveals that DNNs tend to disregard the magnitude of features/perturbations, which highlights a distinctive aspect of DNN robustness. Based on this insight, we utilize the proposed harmless perturbations for hiding sentitive data and model fingerprints.

{\small
\bibliographystyle{abbrvnat}
\bibliography{neurips_2024}
}







\appendix

\section{Proofs of Theorems and Remarks}
In this section, we prove the theorems and remarks in the paper.

\textbf{Theorem 1} (Dimension of harmless perturbation subspace  for a linear layer) 
\label{appx:theorem1}
Given a linear layer $\mathcal{L}(x) = Ax\in\mathbb{R}^{m}$ and an input sample $x\in\mathbb{R}^{n}$, where the parameter matrix $A\in\mathbb{R}^{m\times n}$. The dimension of the subspace for harmless perturbations is $dim(\mathcal{S}) = n-rank(A)$.
\begin{proof} 

The rank-nullity theorem in linear algebra states that, for an $m\times n$ matrix $A$, which represents a linear map $\mathcal{L}: \mathbb{R}^{n} \to \mathbb{R}^{m}$, the number of columns of a matrix $A$
is the sum of the rank of $A$ and the nullity of $A$, \textit{i.e.},
\begin{equation}
rank(A) + nullity(A) = n
\end{equation}
where the rank of $A$ is the dimension of the column space of $A$, \textit{i.e.}, $rank(A) = dim(C(A))$, and the nullity of $A$ is the dimension of the nullspace of $A$, \textit{i.e.}, $nullity(A) = dim(N(A))$.

Here, the harmless perturbation subspace for a linear layer $\mathcal{L}$ is $\mathcal{S}=\{\delta\in\mathbb{R}^{n}|A(x+\delta)=Ax\}=\{\delta|A\delta=\mathbf{0}\} = N(A)$, where the nullspace of the matrix $A$ is $N(A) \coloneqq \{v\in\mathbb{R}^n| Av = \mathbf{0}\}$.

Thus, the dimension of the subspace for harmless perturbations is $dim(\mathcal{S}) = dim(N(A)) = nullity(A) = n-rank(A)$.

\end{proof}

\textbf{Remark 1} 
\label{appx:remark2}
Consider the case that the input dimension of the linear layer is less than or equal to the output dimension, \textit{i.e.}, $n\le m$. In this case, if the column vectors of the parameter matrix $A$ are linearly independent, then the dimension of the subspace for harmless perturbations is $dim(\mathcal{S}) =0$.
\begin{proof} 

According to~\cref{theorem1}, the dimension of the subspace for harmless perturbations is $dim(\mathcal{S}) = n-rank(A)$.

The rank of $A\in\mathbb{R}^{m\times n}$ satisfies $rank(A)\le \min(m,n)$. Consider the case that $n \le m$, then $rank(A)\le n$.

If the column vectors of $A$ are linearly independent, then $rank(A) = n$.

Thus, the dimension of the subspace for harmless perturbations is $dim(\mathcal{S}) = n-n = 0$.

\end{proof}

\textbf{Remark 2} 
\label{appx:remark1}
Consider the case that the input dimension of the linear layer is greater than the output dimension, \textit{i.e.}, $n>m$. In this case, if the row vectors of the parameter matrix $A$ are linearly independent, then the dimension of the subspace for harmless perturbations is $dim(\mathcal{S}) = n-m$.
\begin{proof} 
According to~\cref{theorem1}, the dimension of the subspace for harmless perturbations is $dim(\mathcal{S}) = n-rank(A)$.

The rank of $A\in\mathbb{R}^{m\times n}$ satisfies $rank(A)\le \min(m,n)$. Consider the case that $n>m$, then $rank(A)\le m$.

If the row vectors of $A$ are linearly independent, then $rank(A) = m$.

Thus, the dimension of the subspace for harmless perturbations is $dim(\mathcal{S}) = n-m$.

\end{proof}

\textbf{Theorem 2} (Set of harmless perturbations for two-layer neural networks) \label{appx:theorem_non_linear} 
Given a two-layer neural network $f(x) = \sigma(Ax)$, where $\sigma$ represents any function. If $\sigma$ is an injective function, the set of harmless perturbations on the input $\mathcal{P}$ for $f$ is $\mathcal{P} = \mathcal{P}^{(0)}$. Otherwise, $\mathcal{P} = \mathcal{P}^{(0)}\cup \mathcal{P}^{(1)}\supseteq \mathcal{\mathcal{P}}^{(0)}$. Here, $\mathcal{P}^{(1)}=\{\delta|A\delta = \delta^{(1)}, \forall \delta^{(1)}\in\mathcal{H}^{(1)}\cap C(A)\}$ is determined by the specific function $\sigma$ and the input sample $x$.
\begin{proof}

The set of input perturbations for the $(l+1)$-th layer is defined as $\mathcal{P}^{(l)}\coloneqq \{\delta\in \mathbb{R}^{n_\text{in}}|  z^{(l)}+\delta^{(l)} = (f^{(l)}\circ\cdots\circ f^{(1)})(x+\delta), \forall \delta^{(l)}\in \mathcal{H}^{(l)} \}$ such that the perturbations on the $l$-th intermediate-layer features  have no effect on the network output, \textit{i.e.}, $\delta^{(l)} \in \mathcal{H}^{(l)}$.

According to~\cref{lemma:space_of_harmless_perturbations}, given a two-layer neural network $f = \sigma(Ax)$, the set of harmless input perturbations for the network $f$ with two layers is $\mathcal{P} = \bigcup_{l=0}^{L-1} \mathcal{P}^{(l)} = \mathcal{P}^{(0)}\cup \mathcal{P}^{(1)}\supseteq \mathcal{\mathcal{P}}^{(0)}$, $\mathcal{P}^{(0)} \coloneqq \mathcal{H}^{(0)}$.

Suppose $\sigma$ is an injective function, then  $\mathcal{H}^{(1)}=\{\mathbf{0}\}$, according to~\cref{lemma:injective_function}.  Thus, $\mathcal{P}^{(1)}=\{\delta|z^{(1)} + \mathbf{0} = A(x+\delta)\}=\{\delta|A\delta = \mathbf{0}\}= N(A)$, where $z^{(1)} = Ax$. Therefore, the set of harmless perturbations on the input $\mathcal{P}$ for $\sigma(Ax)$ is $\mathcal{P} = \mathcal{P}^{(0)}\cup \mathcal{P}^{(1)}= \mathcal{P}^{(0)}$.

Otherwise, if $\sigma$ is not an injective function, then $\mathcal{H}^{(1)} \ne \{\mathbf{0}\}$, according to~\cref{lemma:injective_function}. Therefore, the harmless perturbation space \textit{may} expand, \textit{i.e.}, $\mathcal{P}  = \mathcal{P}^{(0)}\cup \mathcal{P}^{(1)}\supseteq \mathcal{\mathcal{P}}^{(0)}$. Then, $\mathcal{P}^{(1)} = \{\delta|z^{(1)} + \delta^{(1)} = A(x+\delta), \forall \delta^{(1)}\in \mathcal{H}^{(1)}\} = \{\delta|A\delta = \delta^{(1)}, \forall \delta^{(1)}\in \mathcal{H}^{(1)}\}$, where $z^{(1)} = Ax$. Notice that the equation $A\delta = \delta^{(1)} (\delta \ne \mathbf{0})$ has a solution (meaning at least one solution) if and only if $\delta^{(1)}$ is in the column space of $A$, \textit{i.e.}, $\delta^{(1)}\in C(A)$. Then, $\mathcal{P}^{(1)}=\{\delta|A\delta = \delta^{(1)}, \forall \delta^{(1)}\in\mathcal{H}^{(1)}, \delta^{(1)}\in C(A)\} = \{\delta|A\delta = \delta^{(1)}, \forall \delta^{(1)}\in\mathcal{H}^{(1)}\cap C(A)\}$. Here $\mathcal{P}^{(1)}$ is determined by the specific function $\sigma$ and the input sample $x$. Please see~\cref{lemma:relu_layer,lemma:softmax_layer,lemma:average_pooling_layer} for specific functions.

\end{proof}

\textbf{Theorem 3} (Arbitrary decomposition of perturbations) 
\label{appx:theorem2}
If there exists a harmless perturbation subspace in the $(l+1)$-th linear layer, \textit{i.e.}, $\mathcal{H}^{(l)}\ne \{\mathbf{0}\}$,  given any perturbation $\forall \delta^{(l)}\in \mathbb{R}^{n^{(l)}}$ and $\delta^{(l)}\notin \mathcal{H}^{(l)}$, it can be arbitrarily decomposed into the sum of a harmless perturbation and a harmful perturbation, \textit{i.e.}, $\delta^{(l)} = \delta^{(l)}_a + \delta^{(l)}_b, \forall \delta^{(l)}_a \in \mathcal{H}^{(l)}$ and $\delta^{(l)}_b \notin \mathcal{H}^{(l)}$. Then, $ f^{(l+1)}(\delta^{(l)})= f^{(l+1)}(\delta^{(l)}_b)$. 
\begin{proof} 
Let the matrix $A \in \mathbb{R}^{n^{(l+1)}\times n^{(l)}}$ with linearly independent rows/columns have an equivalent effect with the parameters of the linear layer, \textit{i.e.}, $z^{(l+1)}=f^{(l+1)}(z^{(l)})=Az^{(l)}\in \mathbb{R}^{n^{(l+1)}}$. 

Let $\mathcal{H}^{(l)}=\{\delta^{(l)}\in \mathbb{R}^{n^{(l)}}|A(z^{(l)}+\delta^{(l)})=Az^{(l)}\}=\{\delta^{(l)}|A\delta^{(l)}=\mathbf{0}\} = N(A)\ne \{\mathbf{0}\}$.

(1) First, we will prove that given any perturbation $\forall \delta^{(l)}\in \mathbb{R}^{n^{(l)}}$ and $\delta^{(l)}\notin \mathcal{H}^{(l)}$, it can be arbitrarily decomposed into the sum of a harmless perturbation and a harmful perturbation, \textit{i.e.}, $\delta^{(l)} = \delta^{(l)}_a + \delta^{(l)}_b, \forall \delta^{(l)}_a \in \mathcal{H}^{(l)}$ and $\delta^{(l)}_b \notin \mathcal{H}^{(l)}$. In other words, it is equivalent to proving that $\forall \delta^{(l)}_a \in \mathcal{H}^{(l)}, \delta^{(l)}_b \coloneqq \delta^{(l)} -  \delta_a^{(l)} \notin \mathcal{H}^{(l)}$.

To achieve this, we prove $\delta^{(l)}_b\notin \mathcal{H}^{(l)}$ by contradiction. Assume that $\delta^{(l)}_b\in \mathcal{H}^{(l)}$, then we can obtain $\forall \delta^{(l)}_a \in \mathcal{H}^{(l)}, f^{(l+1)}(\delta_a^{(l)})=A\delta^{(l)}_a = \mathbf{0}$ and $\forall \delta^{(l)}_b \coloneqq \delta^{(l)} -  \delta_a^{(l)} \in \mathcal{H}^{(l)}, f^{(l+1)}(\delta_b^{(l)})=A\delta^{(l)}_b = \mathbf{0}$. Then, $\forall \delta^{(l)}\notin \mathcal{H}^{(l)}, f^{(l+1)}(\delta^{(l)})= A\delta^{(l)} = A(\delta^{(l)}_a + \delta^{(l)}_b) = A\delta^{(l)}_a + A\delta^{(l)}_b = \mathbf{0}$, which contradicts $\forall \delta^{(l)}\notin \mathcal{H}^{(l)}, f^{(l+1)}(\delta^{(l)})= A\delta^{(l)} \ne \mathbf{0}$.

Thus, given any perturbation $\delta^{(l)}\notin \mathcal{H}^{(l)}$, it can be arbitrarily decomposed into the sum of a harmless perturbation $\delta^{(l)}_a \in \mathcal{H}^{(l)}$ and a harmful perturbation $\delta^{(l)}_b\coloneqq \delta^{(l)} -  \delta_a^{(l)} \notin \mathcal{H}^{(l)}$.

(2) Second, we will prove that $ f^{(l+1)}(\delta^{(l)})= f^{(l+1)}(\delta^{(l)}_b)$.

We have $\forall \delta^{(l)}_a \in \mathcal{H}^{(l)}, f^{(l+1)}(\delta_a^{(l)})=A\delta^{(l)}_a = \mathbf{0}$, and $\forall \delta^{(l)}_b \coloneqq \delta^{(l)} -  \delta_a^{(l)} \notin \mathcal{H}^{(l)}, f^{(l+1)}(\delta_b^{(l)})=A\delta^{(l)}_b \ne \mathbf{0}$.

\begin{equation}
\begin{aligned}
f^{(l+1)}(\delta^{(l)})&= A\delta^{(l)} \\
&= A(\delta^{(l)}_a + \delta^{(l)}_b) \\
&= A\delta^{(l)}_a + A\delta^{(l)}_b \\
&= A\delta^{(l)}_b  = f^{(l+1)}(\delta_b^{(l)})
\end{aligned}
\end{equation}
Thus, the layer output of any perturbation  $\forall \delta^{(l)}\in \mathbb{R}^{n^{(l)}}$ and $\delta^{(l)}\notin \mathcal{H}^{(l)}$  is equivalent to the layer output of its corresponding harmful component $\delta^{(l)}_b\coloneqq \delta^{(l)} - \delta^{(l)}_a\notin \mathcal{H}^{(l)}, \forall \delta^{(l)}_a \in \mathcal{H}^{(l)}$.
\end{proof}

\textbf{Theorem 4} (Orthogonal decomposition of perturbations) 
\label{appx:theorem3}
If there exists a harmless perturbation subspace in the $(l+1)$-th linear layer, \textit{i.e.}, $\mathcal{H}^{(l)}\ne \{\mathbf{0}\}$,  given any perturbation $\forall \delta^{(l)}\in \mathbb{R}^{n^{(l)}}$ and $\delta^{(l)}\notin \mathcal{H}^{(l)}$, it has a unique decomposition $\delta^{(l)} = \delta^{(l)}_{\parallel} + \delta^{(l)}_{\bot}$ such that the parallel component $\delta^{(l)}_{\parallel} = P\delta^{(l)} \in \mathcal{H}^{(l)} $ and the  orthogonal component $\delta^{(l)}_{\bot} = (I-P)\delta^{(l)} \notin \mathcal{H}^{(l)}$. Then, $f^{(l+1)}(\delta^{(l)}_{\parallel})= \mathbf{0}$ and $f^{(l+1)}(\delta^{(l)}) = f^{(l+1)}(\delta^{(l)}_{\bot})$.

\begin{proof} 
Let the matrix $A \in \mathbb{R}^{n^{(l+1)}\times n^{(l)}}$ with linearly independent rows/columns have an equivalent effect with the parameters of the linear layer, \textit{i.e.}, $z^{(l+1)}=f^{(l+1)}(z^{(l)})=Az^{(l)}\in \mathbb{R}^{n^{(l+1)}}$. 

Let $\mathcal{H}^{(l)}=\{\delta^{(l)}\in \mathbb{R}^{n^{(l)}}|A(z^{(l)}+\delta^{(l)})=Az^{(l)}\}=\{\delta^{(l)}|A\delta^{(l)}=\mathbf{0}\} = N(A)\ne \{\mathbf{0}\}$.

(1) First, we will prove that given any perturbation $\forall \delta^{(l)}\in \mathbb{R}^{n^{(l)}}$ and $\delta^{(l)}\notin \mathcal{H}^{(l)}$, it has a \textit{unique} decomposition $\delta^{(l)} = \delta^{(l)}_{\parallel} + \delta^{(l)}_{\bot}$ such that $\delta^{(l)}_{\parallel} = P\delta^{(l)} \in \mathcal{H}^{(l)} $ and  $\delta^{(l)}_{\bot} = (I-P)\delta^{(l)} \notin \mathcal{H}^{(l)}$. 

Here, $P= U(U^\top U)^{-1}U^\top$ represents the projection matrix onto the subspace $\mathcal{H}^{(l)}$ of $\mathbb{R}^{n^{(l)}}$, and the matrix $U\in\mathbb{R}^{n^{(l)}\times dim(\mathcal{H}^{(l)})}$ denotes a set of $dim(\mathcal{H}^{(l)})$ orthogonal bases for the subspace $\mathcal{H}^{(l)}$. 

\textbf{(a) Orthogonal decomposition.} We will prove that $\forall \delta^{(l)}\in \mathbb{R}^{n^{(l)}}$ and $\delta^{(l)}\notin \mathcal{H}^{(l)}$, it has a decomposition $\delta^{(l)} = \delta^{(l)}_{\parallel} + \delta^{(l)}_{\bot}$.

According to~\cref{theorem2},  given any perturbation $\forall \delta^{(l)}\in \mathbb{R}^{n^{(l)}}$ and $\delta^{(l)}\notin \mathcal{H}^{(l)}$, it can be arbitrarily decomposed into
the sum of a harmless perturbation and a harmful perturbation. Then, let $\delta^{(l)}_{\parallel}\in \mathcal{H}^{(l)}$ and $\delta^{(l)}_{\bot}\notin \mathcal{H}^{(l)}$.

Let $U=[u_1, u_2,\cdots, u_{d}]\in\mathbb{R}^{n^{(l)}\times dim(\mathcal{H}^{(l)})}$, where $d=dim(\mathcal{H}^{(l)})$. Here, $u_1, u_2,\cdots, u_{d}\in \mathbb{R}^{n^{(l)}}$ are a set of orthonormal bases that spans the subspace $\mathcal{H}^{(l)}$. 

Since $\delta^{(l)}_{\parallel}\in \mathcal{H}^{(l)}$, it can be represented as a linear combination of orthonormal bases, \textit{i.e.}, $\delta^{(l)}_{\parallel} = \sum_{i=1}^{dim(\mathcal{H}^{(l)})}c_iu_i$, $c_i\in \mathbb{R}$. Rewrite $\delta^{(l)}_{\parallel}$ into matrix form, \textit{i.e.}, $\delta^{(l)}_{\parallel} = Uc$, $c = [c_1, c_2, \cdots, c_d]^\top$.

Since $\delta^{(l)}_{\bot}\coloneqq \delta^{(l)} - \delta^{(l)}_{\parallel} = \delta^{(l)} - Uc$  is orthogonal to the subspace $\mathcal{H}^{(l)}$, $\delta^{(l)}_{\bot}$ is orthogonal to the orthonormal bases $u_1, u_2,\cdots, u_{d}$, respectively. Then, it can derived that $u_1^\top (\delta^{(l)} - Uc) = 0, \cdots, u_d^\top (\delta^{(l)} - Uc) = 0$. Rewrite these equations into matrix form, \textit{i.e.}, $U^\top (\delta^{(l)} - Uc) = \mathbf{0}$. Thus, $c = (U^\top U)^{-1} U^\top \delta^{(l)}$ ($U^\top U$ is invertible, since $U$ has full column rank).

Thus, $\delta^{(l)}_{\parallel} = Uc = U(U^\top U)^{-1} U^\top \delta^{(l)} = P\delta^{(l)}\in \mathcal{H}^{(l)}$ and $\delta^{(l)}_{\bot} = \delta^{(l)} - \delta^{(l)}_{\parallel} = (I-P)\delta^{(l)}\notin \mathcal{H}^{(l)}$. We have proved that $\forall \delta^{(l)}\in \mathbb{R}^{n^{(l)}}$ and $\delta^{(l)}\notin \mathcal{H}^{(l)}$, it has a decomposition $\delta^{(l)} = \delta^{(l)}_{\parallel} + \delta^{(l)}_{\bot}$.

\textbf{(b) Uniqueness of orthogonal decomposition.} We will prove that the orthogonal decomposition on a subspace is \textit{unique}.

Let $\mathcal{H}^{(l)} \subset \mathbb{R}^{n^{(l)}}$ be a subspace of $\mathbb{R}^{n^{(l)}}$, and let $P_1$, $P_2$ be arbitrary projection matrices  onto $\mathcal{H}^{(l)}$, we prove that the orthogonal projector onto $\mathcal{H}^{(l)}$ is unique, \textit{i.e.}, $P_1 = P_2$.

For $\forall \delta^{(l)} \in \mathbb{R}^{n^{(l)}}$, the parallel components are $\delta^{(l)}_{\parallel, 1} = P_1\delta^{(l)}\in \mathcal{H}^{(l)}$ and $\delta^{(l)}_{\parallel, 2} = P_2\delta^{(l)}\in \mathcal{H}^{(l)}$. The corresponding orthogonal components satisfy $\forall \delta^{(l)} \in \mathbb{R}^{n^{(l)}},  (\delta^{(l)} - P_1\delta^{(l)})\bot \mathcal{H}^{(l)}$ and $(\delta^{(l)} - P_2\delta^{(l)})\bot \mathcal{H}^{(l)}$, thus, $(P_1- P_2)\delta^{(l)} \bot \mathcal{H}^{(l)}$. However, $(P_1- P_2)\delta^{(l)} = \delta^{(l)}_{\parallel, 1} - \delta^{(l)}_{\parallel, 2} \in \mathcal{H}^{(l)}$, then we have $(P_1- P_2)\delta^{(l)} = \mathbf{0}$ for every $\delta^{(l)}$. Therefore, $P_1 = P_2$. 

Thus, given any perturbation $\forall \delta^{(l)}\in \mathbb{R}^{n^{(l)}}$ and $\delta^{(l)}\notin \mathcal{H}^{(l)}$, the orthogonal decomposition of perturbation $\delta^{(l)} = \delta^{(l)}_{\parallel} + \delta^{(l)}_{\bot}$ on a subspace $\mathcal{H}^{(l)}\subset \mathbb{R}^{n^{(l)}}$ is unique.

(2) Second, we will prove that $f^{(l+1)}(\delta^{(l)}_{\parallel})= \mathbf{0}$ and $f^{(l+1)}(\delta^{(l)}) = f^{(l+1)}(\delta^{(l)}_{\bot})$.

We have $ \delta^{(l)}_{\parallel} \in \mathcal{H}^{(l)}, f^{(l+1)}(\delta_{\parallel}^{(l)})=A\delta^{(l)}_{\parallel} = \mathbf{0}$, and $ \delta^{(l)}_{\bot} \coloneqq \delta^{(l)} -  \delta_{\parallel}^{(l)} \notin \mathcal{H}^{(l)}, f^{(l+1)}(\delta_{\bot}^{(l)})=A\delta^{(l)}_{\bot} \ne \mathbf{0}$.

\begin{equation}
\begin{aligned}
f^{(l+1)}(\delta^{(l)})&= A\delta^{(l)} \\
&= A(\delta^{(l)}_{\parallel} + \delta^{(l)}_{\bot}) \\
&= A\delta^{(l)}_{\parallel} + A\delta^{(l)}_{\bot} \\
&= A\delta^{(l)}_{\bot}  = f^{(l+1)}(\delta_{\bot}^{(l)})
\end{aligned}
\end{equation}

Thus, the layer output of any perturbation  $\forall \delta^{(l)}\in \mathbb{R}^{n^{(l)}}$ and $\delta^{(l)}\notin \mathcal{H}^{(l)}$  is equivalent to the layer output of its unique orthogonal component $\delta^{(l)}_{\bot}\coloneqq \delta^{(l)} - \delta^{(l)}_{\parallel} = (I-P)\delta^{(l)} \notin \mathcal{H}^{(l)}$.

\end{proof}

\textbf{Theorem 5} (Identical impact of a family of perturbations) 
\label{appx:theorem4}
If there exists a harmless perturbation subspace in the $(l+1)$-th linear layer, \textit{i.e.}, $\mathcal{H}^{(l)}\ne \{\mathbf{0}\}$, given two different perturbations $\forall \delta^{(l)}\ne \hat{\delta}^{(l)} \in \mathbb{R}^{n^{(l)}}$ and $\delta^{(l)}, \hat{\delta}^{(l)} \notin \mathcal{H}^{(l)}$, if their orthogonal components are the same, \textit{i.e.}, $\delta_{\bot}^{(l)} = \hat{\delta}_{\bot}^{(l)}$, then $f^{(l+1)}(\delta^{(l)}) =  f^{(l+1)}(\hat{\delta}^{(l)})$.

\begin{proof} 
According to~\cref{theorem3}, $\forall \delta^{(l)}\in \mathbb{R}^{n^{(l)}}$ and $\delta^{(l)}\notin \mathcal{H}^{(l)}, f^{(l+1)}(\delta^{(l)}) = f^{(l+1)}(\delta_{\bot}^{(l)})$.



Similarly, $\forall \hat{\delta}^{(l)} \ne \delta^{(l)} \in \mathbb{R}^{n^{(l)}}$ and $\hat{\delta}^{(l)}\notin \mathcal{H}^{(l)}, f^{(l+1)}(\hat{\delta}^{(l)}) = f^{(l+1)}(\hat{\delta}_{\bot}^{(l)})$.

Thus, given $\delta_{\bot}^{(l)} = \hat{\delta}_{\bot}^{(l)}$, then $f^{(l+1)}(\delta^{(l)}) = f^{(l+1)}(\delta_{\bot}^{(l)}) = f^{(l+1)}(\hat{\delta}_{\bot}^{(l)}) = f^{(l+1)}(\hat{\delta}^{(l)})$.

\end{proof}

\section{Proofs of Corollaries}
In this section, we prove the corollaries in the paper.

\textbf{Corollary 1} (Dimension of harmless perturbation subspace for a convolutional layer) 
\label{appx:corollary1}
Given a convolutional layer $f^{(l+1)}$ with linearly independent vectorized kernels whose kernel size is larger than or equal to the stride, the input feature is $z^{(l)}\in \mathbb{R}^{C_{\text{\rm in}} \times H_{\text{\rm in}} \times W_{\text{\rm in}}}$ and the output feature is $z^{(l+1)}=f^{(l+1)}(z^{(l)})\in\mathbb{R}^{C_{\text{\rm out}} \times H_{\text{\rm out}} \times W_{\text{\rm out}}}$. If the input dimension of the convolutional layer is greater than the output dimension, then the dimension of the subspace for harmless perturbations is $dim(\mathcal{H}^{(l)}) = C_{\text{\rm in}}  H_{\text{\rm in}}  W_{\text{\rm in}} - C_{\text{\rm out}}  H_{\text{\rm out}} W_{\text{\rm out}}$. Otherwise, $\mathcal{H}^{(l)} = \{\mathbf{0}\}$.

\begin{proof} 

Let the matrix $A\in \mathbb{R}^{(C_{\text{\rm out}}  H_{\text{\rm out}} W_{\text{\rm out}})\times  (C_{\text{\rm in}}  H_{\text{\rm in}}  W_{\text{\rm in}})}$ with linearly independent rows/columns (see~\cref{appx:equivalent_matrix} for the generation of matrix $A$) have an equivalent effect with the parameters of the convolutional layer, \textit{i.e.}, $\hat{z}^{(l+1)}=A\hat{z}^{(l)}$, $\hat{z}^{(l)}\in \mathbb{R}^{C_{\text{\rm in}}  H_{\text{\rm in}}  W_{\text{\rm in}}}$ is the vectorized feature $z^{(l)}$, and $\hat{z}^{(l+1)}\in \mathbb{R}^{C_{\text{\rm out}}  H_{\text{\rm out}}  W_{\text{\rm out}}}$ is the vectorized feature $z^{(l+1)}$. 

Let $\mathcal{H}^{(l)}=\{\delta^{(l)}\in\mathbb{R}^{C_{\text{\rm in}}  H_{\text{\rm in}}  W_{\text{\rm in}}}|A(\hat{z}^{(l)}+\delta^{(l)})=A\hat{z}^{(l)}\}=\{\delta^{(l)}|A\delta^{(l)}=\mathbf{0}\} = N(A)$.

According to~\cref{remark2}, if $C_{\text{\rm in}}  H_{\text{\rm in}}  W_{\text{\rm in}} \le C_{\text{\rm out}}  H_{\text{\rm out}} W_{\text{\rm out}}$ and the column vectors of $A$ are linearly independent, then the dimension of the subspace for harmless perturbations is $dim(\mathcal{H}^{(l)}) = C_{\text{\rm in}}  H_{\text{\rm in}}  W_{\text{\rm in}}-C_{\text{\rm in}}  H_{\text{\rm in}}  W_{\text{\rm in}}=0$. Thus, $\mathcal{H}^{(l)} = \{\mathbf{0}\}$.

According to Remark 2 in~\cref{appx:remark1}, if $C_{\text{\rm in}}  H_{\text{\rm in}}  W_{\text{\rm in}} > C_{\text{\rm out}}  H_{\text{\rm out}} W_{\text{\rm out}}$ and the row vectors of $A$ are linearly independent, then the dimension of the subspace for harmless perturbations is $dim(\mathcal{H}^{(l)}) = C_{\text{\rm in}}  H_{\text{\rm in}}  W_{\text{\rm in}}-C_{\text{\rm out}}  H_{\text{\rm out}} W_{\text{\rm out}}$.

Thus, if the input dimension of the convolutional layer is greater than the output dimension, \textit{i.e.}, $C_{\text{\rm in}}  H_{\text{\rm in}}  W_{\text{\rm in}} > C_{\text{\rm out}}  H_{\text{\rm out}} W_{\text{\rm out}}$, then the dimension of the subspace for harmless perturbations is $dim(\mathcal{H}^{(l)}) = C_{\text{\rm in}}  H_{\text{\rm in}}  W_{\text{\rm in}} - C_{\text{\rm out}}  H_{\text{\rm out}} W_{\text{\rm out}}$. Otherwise, $\mathcal{H}^{(l)} = \{\mathbf{0}\}$.

\end{proof}

\begin{corollary}[Dimension of harmless perturbation subspace for a fully-connected layer]
\label{appx:corollary2}
Given a fully-connected layer $z^{(l+1)}=f^{(l+1)}(z^{(l)}) = W^\top z^{(l)}\in \mathbb{R}^{N_{\text{\rm out}}}$ with linearly independent rows/columns in the matrix $W$, the input feature is $z^{(l)}\in \mathbb{R}^{N_{\text{\rm in}}}$. If the input dimension of the fully-connected layer is greater than the output dimension, then the dimension of the subspace for harmless perturbations is $dim(\mathcal{H}^{(l)}) = N_{\text{\rm in}} - N_{\text{\rm out}}$. Otherwise, $\mathcal{H}^{(l)} = \{\mathbf{0}\}$.  
\end{corollary}

\begin{proof} 
Let the matrix $A=W^\top \in \mathbb{R}^{N_{\text{\rm out}}\times N_{\text{\rm in}}}$ with linearly independent rows/columns have an equivalent effect with the parameters of the fully-connected layer. Let $\mathcal{H}^{(l)}=\{\delta^{(l)}\in\mathbb{R}^{ N_{\text{\rm in}}}|A(z^{(l)}+\delta^{(l)})=Az^{(l)}\}=\{\delta^{(l)}|A\delta^{(l)}=\mathbf{0}\} = N(A)$.

According to~\cref{remark2}, if $N_{\text{\rm in}} \le N_{\text{\rm out}}$ and the column vectors of $A$ are linearly independent, then the dimension of the subspace for harmless perturbations is $dim(\mathcal{H}^{(l)}) = N_{\text{\rm in}}-N_{\text{\rm in}}=0$. Thus, $\mathcal{H}^{(l)} = \{\mathbf{0}\}$.

According to Remark 2 in~\cref{appx:remark1}, if $N_{\text{\rm in}} > N_{\text{\rm out}}$ and the row vectors of $A$ are linearly independent, then the dimension of the subspace for harmless perturbations is $dim(\mathcal{H}^{(l)}) = N_{\text{\rm in}}-N_{\text{\rm out}}$.

Thus, if the input dimension of the fully-connected layer is greater than the output dimension, \textit{i.e.}, $N_{\text{\rm in}} > N_{\text{\rm out}}$, then the dimension of the subspace for harmless perturbations is $dim(\mathcal{H}^{(l)}) = N_{\text{\rm in}} - N_{\text{\rm out}}$. Otherwise, $\mathcal{H}^{(l)} = \{\mathbf{0}\}$. 

\end{proof}

\section{Proofs of Lemmas}
In this section, we prove the lemmas in the paper.

\textbf{Lemma 1} (Set of harmless input perturbations for a DNN) 
\label{appx:lemma1}
The set of harmless input perturbations for  a DNN $f$ with $L$ layers is derived as $\mathcal{P} = \bigcup_{l=0}^{L-1} \mathcal{P}^{(l)}$
    , $\mathcal{P}^{(0)} \coloneqq \mathcal{H}^{(0)}$, $\mathcal{P}\subset \mathbb{R}^{n_\text{in}}$.

\begin{proof} 
The set of harmelss perturbations for a DNN $f$ is defined as $\mathcal{P}\coloneqq \{\delta\in \mathbb{R}^{n_\text{in}} | f(x+\delta) = f(x)\}$.

The set of harmless perturbations on the $(l+1)$-th layer of a DNN $f$ is defined as $\mathcal{H}^{(l)} \coloneqq \{\delta^{(l)} \in \mathbb{R}^{n^{(l)}} | f^{(l+1)}(z^{(l)}+\delta^{(l)}) = f^{(l+1)}(z^{(l)})\}$. Here, $z^{(l)}\in \mathbb{R}^{n^{(l)}}$ represents the features of the $l$-th intermediate layer of the input sample $x$, and $\delta^{(l)}$ denotes the perturbations added to the features $z^{(l)}$.

The set of input perturbations for the $(l+1)$-th layer is defined as $\mathcal{P}^{(l)}\coloneqq \{\delta\in \mathbb{R}^{n_\text{in}}|  z^{(l)}+\delta^{(l)} = (f^{(l)}\circ\cdots\circ f^{(1)})(x+\delta), \forall \delta^{(l)}\in \mathcal{H}^{(l)} \}$ such that the perturbations on the $l$-th intermediate-layer features  have no effect on the network output, \textit{i.e.}, $\delta^{(l)} \in \mathcal{H}^{(l)}$.

To prove that $\mathcal{P} = \bigcup_{l=0}^{L-1} \mathcal{P}^{(l)}$, We will prove that (1) $\forall \delta \in \bigcup_{l=0}^{L-1} \mathcal{P}^{(l)}, \delta \in \mathcal{P}$, and (2) $\forall \delta \in \mathcal{P}, \delta \in \bigcup_{l=0}^{L-1} \mathcal{P}^{(l)}$, respectively. Notice that the $(l+1)$-th layer can be arbitrary layer (linear or non-linear layer), $\mathcal{P}$, $\mathcal{H}^{(l)}$ and $\mathcal{P}^{(l)}$ are instance-specific, given a specific  input sample $x$ and the corresponding feature $z^{(l)}$ ($x \coloneqq z^{(0)}$).

(1) We will prove that $\forall \delta \in \bigcup_{l=0}^{L-1} \mathcal{P}^{(l)}, \delta \in \mathcal{P}$.

Given an input sample $x$, $\forall \delta \in \bigcup_{l=0}^{L-1} \mathcal{P}^{(l)}$ means that $\forall \delta, \exists l\in \{0,1,\cdots, L-1\}$ such that $ \delta \in \mathcal{P}^{(l)}\coloneqq \{\delta\in \mathbb{R}^{n_\text{in}}|  z^{(l)}+\delta^{(l)} = (f^{(l)}\circ\cdots\circ f^{(1)})(x+\delta), \forall \delta^{(l)}\in \mathcal{H}^{(l)} \}$, $\mathcal{P}^{(0)} \coloneqq \mathcal{H}^{(0)}$. Then, the corresponding harmless perturbations  $\delta^{(l)}$ on the $l$-th intermediate-layer features satisfy $ f^{(l+1)}(z^{(l)}+\delta^{(l)}) = f^{(l+1)}(z^{(l)})$ due to $\delta^{(l)}\in \mathcal{H}^{(l)}$. Therefore, the layer outputs of the subsequent layers are all equal, $\forall l' \in \{l, l+1, \cdots, L-1\}, f^{(l'+1)}(z^{(l')}+\delta^{(l')}) = f^{(l'+1)}(z^{(l')})$ and finally $f^{(L)}(z^{(L-1)}+\delta^{(L-1)}) = f^{(L)}(z^{(L-1)})$. Thus, $f(x+\delta) = f(x)$ can be derived and hence $\delta \in \mathcal{P}\coloneqq \{\delta\in \mathbb{R}^{n_\text{in}} | f(x+\delta) = f(x)\}$. 

Thus, we have proved that $\forall \delta \in \bigcup_{l=0}^{L-1} \mathcal{P}^{(l)}, \delta \in \mathcal{P}$.

(2) We will prove that $\forall \delta \in \mathcal{P}, \delta \in \bigcup_{l=0}^{L-1} \mathcal{P}^{(l)}$.

We will prove $\forall \delta \in \mathcal{P}, \delta \in \bigcup_{l=0}^{L-1} \mathcal{P}^{(l)}$ by contradiction. Assume that $\exists \delta \in \mathcal{P}, \delta \notin \bigcup_{l=0}^{L-1} \mathcal{P}^{(l)}$, $\mathcal{P}^{(0)} \coloneqq \mathcal{H}^{(0)}$. Here, $\delta \notin \bigcup_{l=0}^{L-1} \mathcal{P}^{(l)}$ means that $\forall l\in \{0,1,\cdots, L-1\}, \delta \notin \mathcal{P}^{(l)}$. Then, the corresponding harmless perturbations  $\delta^{(l)}$ on the $l$-th intermediate-layer features satisfy $\forall l\in \{0,1,\cdots, L-1\}, \delta^{(l)} \notin \mathcal{H}^{(l)}\coloneqq \{\delta^{(l)} \in \mathbb{R}^{n^{(l)}} | f^{(l+1)}(z^{(l)}+\delta^{(l)}) = f^{(l+1)}(z^{(l)})\}$. Therefore, the outputs of each layer are not equal, \textit{i.e.}, $\forall l\in \{0,1,\cdots, L-1\}, f^{(l+1)}(z^{(l)}+\delta^{(l)}) \ne f^{(l+1)}(z^{(l)})$ and finally $f^{(L)}(z^{(L-1)}+\delta^{(L-1)}) \ne f^{(L)}(z^{(L-1)})$.Thus, $f(x+\delta) \ne f(x)$, which is in contradiction to $\delta \in \mathcal{P}\coloneqq \{\delta\in \mathbb{R}^{n_\text{in}} | f(x+\delta) = f(x)\}$. 

Thus, we have proved that $\forall \delta \in \mathcal{P}, \delta \in \bigcup_{l=0}^{L-1} \mathcal{P}^{(l)}$.

Since (1) $\forall \delta \in \bigcup_{l=0}^{L-1} \mathcal{P}^{(l)}, \delta \in \mathcal{P}$ and (2) $\forall \delta \in \mathcal{P}, \delta \in \bigcup_{l=0}^{L-1} \mathcal{P}^{(l)}$, we have proved that $\mathcal{P} = \bigcup_{l=0}^{L-1} \mathcal{P}^{(l)}$, $\mathcal{P}\subset \mathbb{R}^{n_\text{in}}$.

\end{proof}

\textbf{Lemma 1.1} (Set of harmless perturbations for injective functions)
\label{appx:injective_function}
If the $\!(l\!+\!1\!)$-th layer $f^{(l+1)}$ is an injective function, the set of harmless perturbations on the $\!(l\!+\!1\!)$-th layer of a DNN $f$ is $\mathcal{H}^{(l)} = \{\mathbf{0}\}$. Otherwise, $\mathcal{H}^{(l)} \ne \{\mathbf{0}\}$.
\begin{proof} 
According to~\cref{def:set_of_harmless_perturbations_features}, the set of harmless perturbations on the $(l+1)$-th layer of a DNN $f$ is defined as $\mathcal{H}^{(l)} \coloneqq \{\delta^{(l)} \in \mathbb{R}^{n^{(l)}} | f^{(l+1)}(z^{(l)}+\delta^{(l)}) = f^{(l+1)}(z^{(l)})\}$.

If $f^{(l+1)}$ is an injective function that maps distinct elements of its domain to distinct elements, that is, $x_1\ne x_2$ implies $f^{(l+1)}(x_1)\ne f^{(l+1)}(x_2)$. Since $\forall \delta^{(l)}\ne \mathbf{0}, z^{(l)}+\delta^{(l)}\ne z^{(l)}$, the function output satisfies $f^{(l+1)}(z^{(l)}+\delta^{(l)}) \ne f^{(l+1)}(z^{(l)})$, then, $\mathcal{H}^{(l)}=\{\mathbf{0}\}$.  

Otherwise, if $f^{(l+1)}$ is not an injective function, that is,  $\exists x_1\ne x_2$ satisfies $f^{(l+1)}(x_1) =  f^{(l+1)}(x_2)$. Since $\exists \delta^{(l)}\ne \mathbf{0}, z^{(l)}+\delta^{(l)}\ne z^{(l)}$, the function output satisfies $f^{(l+1)}(z^{(l)}+\delta^{(l)}) = f^{(l+1)}(z^{(l)})$, then, $\mathcal{H}^{(l)} \ne \{\mathbf{0}\}$.

\end{proof}

\textbf{Lemma 1.2} (Set of harmless perturbations for ReLU layers) \label{appx:relu_layer}
Suppose $f^{(l+1)}$ is the \text{\rm ReLU} layer, $\mathcal{H}^{(l)} = \{ \delta^{(l)}|\forall i, \delta^{(l)}_i =\begin{cases} 0,& z^{(l)}_i>0\\ t(t\le-z^{(l)}_i),&  z^{(l)}_i\le0 \end{cases} $$\}$, which is determined by intermediate-layer features $z^{(l)}$ and hence the input sample $x$.

\begin{proof} 

Suppose $f^{(l+1)}$ is the ReLU layer, which is not an injective function. The set of harmless perturbations for the ReLU layer $f^{(l+1)}$ is $\mathcal{H}^{(l)}= \{\delta^{(l)}|\text{\rm ReLU}(z^{(l)}+\delta^{(l)})=\text{\rm ReLU}(z^{(l)})\}$. Since the ReLU layer outputs 0 for all inputs that are not positive, then the $i$-th element (dimension) of $\delta^{(l)}$ in $\mathcal{H}^{(l)}$ satisfies, 

\begin{equation}
\forall i, \delta^{(l)}_i =\begin{cases} 0,& z^{(l)}_i>0\\ t(t\le-z^{(l)}_i),&  z^{(l)}_i\le0 \end{cases} 
\end{equation}

Here, $t$ denotes any real number not greater than $-z^{(l)}_i$. In other words, each element of $\delta^{(l)}$ in $\mathcal{H}^{(l)}$ is determined by the value of the corresponding dimension of the intermediate-layer feature $z^{(l)}$. Some dimensions of  $\delta^{(l)}$ are zero, while others are not greater than the corresponding dimensions of the intermediate-layer feature's (negative) value.

Therefore, suppose $f^{(l+1)}$ is the ReLU layer, $\mathcal{H}^{(l)} = \{ \delta^{(l)}|\forall i, \delta^{(l)}_i =\begin{cases} 0,& z^{(l)}_i>0\\ t(t\le-z^{(l)}_i),&  z^{(l)}_i\le0 \end{cases} $$\}$, which is determined by intermediate-layer features $z^{(l)}$ and hence the input sample $x$.

Besides, let us further consider the impact of the input sample $x$ on the harmless perturbation space of the ReLU layer. Specifically, let us consider the case of the harmless perturbation space for a two-layer neural network with the ReLU layer, \textit{i.e.}, $f(x)=\text{\rm ReLU}(Ax)$. 

According to~\cref{lemma:relu_layer} and~\cref{theorem_non_linear}, $\mathcal{P} = \mathcal{P}^{(0)}\cup \mathcal{P}^{(1)}\supseteq \mathcal{\mathcal{P}}^{(0)}$ and $\mathcal{P}^{(1)}=\{\delta|A\delta = \delta^{(1)}, \forall \delta^{(1)}\in\mathcal{H}^{(1)}\cap C(A)\}$. Note that $\mathcal{P}^{(1)}$ is determined by the input sample $x$. In an extreme case, if every element of $z^{(1)}=Ax$ is positive, $\mathcal{H}^{(1)} = \{\mathbf{0}\}$ and $\mathcal{P}^{(1)}=\{\delta|A\delta = \mathbf{0}\}=N(A)$. Thus, the set of harmless perturbations on the input $\mathcal{P}$ for $\text{\rm ReLU}(Ax)$ is $\mathcal{P} = \mathcal{P}^{(0)}\cup \mathcal{P}^{(1)}= \mathcal{P}^{(0)}$. Otherwise, if every element of $z^{(1)}=Ax$ is not positive, then $\mathcal{H}^{(1)} =\{\delta^{(1)}|\forall i, \delta^{(1)}_i\le-z^{(1)}_i\}$ and $\mathcal{P}^{(1)}=\{\delta|A\delta = \delta^{(1)}, \forall \delta^{(1)}\in\mathcal{H}^{(1)}\cap C(A)\}$. Thus, the set of harmless perturbations on the input $\mathcal{P}$ for $\text{\rm ReLU}(Ax)$ is $\mathcal{P} = \mathcal{P}^{(0)}\cup \mathcal{P}^{(1)}\supseteq \mathcal{P}^{(0)}$.

\end{proof}

\textbf{Lemma 1.3} (Set of harmless perturbations for Softmax layers) \label{appx:softmax_layer}
Suppose $f^{(l+1)}$ is the \text{\rm Softmax} layer,  $\mathcal{H}^{(l)}=\{c\cdot\mathbf{1}, c\in\mathbb{R}\}$.

\begin{proof}

Suppose $f^{(l+1)}$ is the Softmax layer, which is not an injective function. The set of harmless perturbations for the Softmax layer $f^{(l+1)}$ is $\mathcal{H}^{(l)}= \{\delta^{(l)}|\text{\rm Softmax}(z^{(l)}+\delta^{(l)})=\text{\rm Softmax}(z^{(l)})\}$. Since the Softmax layer is invariant when translating the same value in each coordinate, that is, adding $\mathbf{c}=(c,c,\cdots, c)$ to the input yields $\text{\rm Softmax}(\mathbf{x} + \mathbf{c}) = \text{\rm Softmax}(\mathbf{x})$, because the $i$-th element (dimension) satisfies, 

\begin{equation}\forall i, \text{\rm Softmax}(\mathbf{x} + \mathbf{c})_i = \frac{e^{x_i + c}}{\sum_{k=1}^{K}e^{x_k + c}} = \frac{e^{x_i}\cdot e^{c}}{\sum_{k=1}^{K}e^{x_k}\cdot e^{c}} = \text{\rm Softmax}(\mathbf{x})_i \end{equation}

Therefore, suppose $f^{(l+1)}$ is the \text{\rm Softmax} layer,  $\mathcal{H}^{(l)}=\{c\cdot\mathbf{1}, c\in\mathbb{R}\}$.

\end{proof}

\textbf{Lemma 1.4} (Set of harmless perturbations for Average Pooling layers) \label{appx:average_pooling_layer}
Suppose $f^{(l+1)}$ is the \text{\rm Average Pooling} layer, $\mathcal{H}^{(l)}=N(A_{\text{\rm avg}})$. $A_{\text{\rm avg}}$ is a coefficient matrix determined by the constraints that must be satisfied by the perturbations within each averaging region.

\begin{proof}

Suppose $f^{(l+1)}$ is the Average Pooling layer, which is not an injective function. The set of harmless perturbations for the Average Pooling layer $f^{(l+1)}$ is $\mathcal{H}^{(l)}= \{\delta^{(l)}|\text{\rm AvgPool}(z^{(l)}+\delta^{(l)})=\text{\rm AvgPool}(z^{(l)})\}$.

Notice that the Average Pooling layer computes the average over each $k \times k$ region of features. Thus, within this $k \times k$ region, each element (dimension) of $\delta^{(l)}$ in $\mathcal{H}^{(l)}$ must satisfy the equation:

\begin{equation}z^{(l)}_1 + z^{(l)}_2 + … + z^{(l)}_{k\times k} = (z^{(l)}_1 + \delta^{(l)}_1) + (z^{(l)}_2 + \delta^{(l)}_2) + … + (z^{(l)}_{k\times k} + \delta^{(l)}_{k\times k})\end{equation}

Consequently, the perturbations within each  $k \times k$  region must satisfy: $\delta^{(l)}_1 + \delta^{(l)}_2 +\cdots + \delta^{(l)}_{k\times k} = 0$.

Therefore, harmless perturbations within all $K\times K$ region  (\textit{i.e.}, shape of feature is $K\times K$)  can be represented as $A_{\text{avg}}\delta^{(l)} = \mathbf{0}$, where each row of $A_{\text{avg}}$ represents the constraint that perturbations within each $k \times k$ region must satisfy. Thus, $\mathcal{H}^{(l)}=\{\delta^{(l)}|A_{\text{avg}}\delta^{(l)} = \mathbf{0}\} = N(A_{\text{avg}})$.

\end{proof}

\textbf{Lemma 1.5} (Set of harmless perturbations for Max Pooling layers) \label{appx:max_pooling_layer}
Suppose $f^{(l+1)}$ is the \text{\rm Max Pooling} layer, $\mathcal{H}^{(l)}=\{\forall p,i, \delta^{(l)}_{p,i} \le c_p-z^{(l)}_{p,i}\} \cap \{\forall p, \prod_{j=1}^{k\times k} (\delta^{(l)}_{p,j}-c_p+z^{(l)}_{p,j}) = 0\}$. $c_p \coloneqq  \textit{\rm max} \{z^{(l)}_{p,1}, z^{(l)}_{p,2}, \cdots, z^{(l)}_{p,k\times k} \}$ is the maximum value of features within the $k\times k$ region of the $p$-th patch. $\mathcal{H}^{(l)}$ is determined by intermediate-layer features $z^{(l)}$ and hence the input sample $x$.
\begin{proof}
Suppose $f^{(l+1)}$ is the Max Pooling layer, which is not an injective function. The set of harmless perturbations for the Max Pooling layer $f^{(l+1)}$ is $\mathcal{H}^{(l)}= \{\delta^{(l)}|\text{\rm MaxPool}(z^{(l)}+\delta^{(l)})=\text{\rm MaxPool}(z^{(l)})\}$. 

Notice that the Max Pooling layer computes the maximum value of the feature $z^{(l)}$ within each $k \times k$ region. Let us divide the feature $z^{(l)}$ into $P$ patches. For the $p$-th patch of the feature $z^{(l)} (p=\{1, 2, \cdots, P\})$, we define the maximum value within the $k \times k$ region as $c_p \coloneqq  \textit{\rm max} \{z^{(l)}_{p,1}, z^{(l)}_{p,2}, \cdots, z^{(l)}_{p,k\times k} \}$, which depends on the specific feature $z^{(l)}$ and the input sample $x$. Thus, for the $p$-th patch with $k \times k$ region, each element (dimension) of $\delta^{(l)}$ in $\mathcal{H}^{(l)}$ must satisfy the equation:

\begin{equation}
   \begin{cases}
z^{(l)}_{p,1}+\delta^{(l)}_{p,1} \le c_p \\
z^{(l)}_{p,2}+\delta^{(l)}_{p,2} \le c_p \\
\cdots \\
z^{(l)}_{p,k\times k}+\delta^{(l)}_{p,k \times k} \le c_p
\end{cases}
\end{equation}

Note that for the above $k \times k$ inequalities, it needs to be satisfied that at least one of the inequalities takes equality, such that the maximum value of perturbations added to the features within this patch is still $c_p$, \textit{i.e.}, $\textit{\rm max} \{z^{(l)}_{p,1}+\delta^{(l)}_{p,1}, z^{(l)}_{p,2}+\delta^{(l)}_{p,2}, \cdots, z^{(l)}_{p,k\times k}+\delta^{(l)}_{p,k \times k} \} = c_p$. Hence, an additional equality constraint $\prod_{j=1}^{k\times k} [\delta^{(l)}_{p,j}-(c_p-z^{(l)}_{p,j})] = 0$ needs to be satisfied.

Therefore, for the $p$-th patch with $k \times k$ region, the harmless perturbations satisfy $\{\forall i, \delta^{(l)}_{p,i} \le c_p-z^{(l)}_{p,i}\} \cap \{\prod_{j=1}^{k\times k} (\delta^{(l)}_{p,j}-c_p+z^{(l)}_{p,j}) = 0\}$. For harmless perturbations within all $P$ patches, $\mathcal{H}^{(l)}=\{\forall p,i, \delta^{(l)}_{p,i} \le c_p-z^{(l)}_{p,i}\} \cap \{\forall p, \prod_{j=1}^{k\times k} (\delta^{(l)}_{p,j}-c_p+z^{(l)}_{p,j}) = 0\}$.

\end{proof}

\textbf{Lemma 1.6} (Set of harmless perturbations for two-layer linear networks)\label{appx:two_layer_neural_networks}
Given a two-layer linear network $f(x) = A_2A_1x$,  $\mathcal{P} = \mathcal{P}^{(0)}\cup \mathcal{P}^{(1)}\supseteq \mathcal{P}^{(0)} $. Here, $\mathcal{P}^{(0)}=N(A_1)$ and $\mathcal{P}^{(1)}=\{\delta|A_1\delta = \delta^{(1)},\forall \delta^{(1)}\in N(A_2)\cap C(A_1)\}$.

\begin{proof}

In general, $\mathcal{H}^{(1)} = \{\delta^{(1)}|A_2(z^{(1)}+\delta^{(1)})=A_2z^{(1)}\}=N(A_2)$

$\mathcal{P}^{(1)}=\{\delta|z^{(1)}+\delta^{(1)} = A_1(x+\delta),\forall \delta^{(1)}\in\mathcal{H}^{(1)}\}=\{\delta|A_1\delta = \delta^{(1)},\forall \delta^{(1)}\in N(A_2)\}=\{\delta|A_1\delta = \delta^{(1)},\forall \delta^{(1)}\in N(A_2)\cap C(A_1)\}$, where $z^{(1)}=A_1x$. Note that the equation $A_1\delta = \delta^{(1)}$ has a solution (meaning at least one solution) if and only if $\delta^{(1)}$ is in the column space of $A_1$, \textit{i.e.}, $\delta^{(1)}\in C(A_1)$.

$\mathcal{H}^{(0)}= \{\delta|A_1(x+\delta)=A_1x\}=N(A_1)$

$\mathcal{P}^{(0)}\coloneqq \mathcal{H}^{(0)}$

Therefore, $\mathcal{P} = \mathcal{P}^{(0)}\cup\mathcal{P}^{(1)}\supseteq \mathcal{P}^{(0)}$.

Specifically, let us consider two common cases in DNNs. Given $A_1\in\mathbb{R}^{d \times n}$ and $A_2\in\mathbb{R}^{m \times d}$, where $y = A_2 A_1x \in \mathbb{R}^{m}$ and $x\in \mathbb{R}^{n} $.

(1) $n>d$ and $d < m$, which means the input dimension is greater than the first layer's feature dimension, and the first layer's feature dimension is smaller than the second layer's feature dimension.

According to~\cref{theorem1}, $\mathcal{H}^{(1)}=\{\mathbf{0}\}$ and $\mathcal{H}^{(0)}=N(A_1)$. Then, $\mathcal{P}^{(1)}=\mathcal{H}^{(0)}$ and $\mathcal{P} = \mathcal{P}^{(0)}\cup\mathcal{P}^{(1)}=\mathcal{P}^{(0)}$.

(2) $n < d$ and $d > m$, which means the input dimension is smaller than the first layer's feature dimension, and the first layer's feature dimension is greater than the second layer's feature dimension.

According to~\cref{theorem1}, $\mathcal{H}^{(1)}=N(A_2)$ and $\mathcal{H}^{(0)}=\{\mathbf{0}\}$. Then, $\mathcal{P}^{(1)}=\{\delta|A_1\delta = \delta^{(1)},\forall \delta^{(1)}\in N(A_2)\cap C(A_1)\}$ and $\mathcal{P} = \mathcal{P}^{(0)}\cup\mathcal{P}^{(1)}=\mathcal{P}^{(1)}$.

\end{proof}

\setcounter{lemma}{1} 
\counterwithin*{lemma}{part}
\renewcommand{\thelemma}{\arabic{lemma}}

\begin{lemma}[The least harmful perturbation for a linear layer]
\label{lemma:least_harmful}
If the input dimension of a given linear layer $f^{(l+1)}(z^{(l)}) = Az^{(l)}$ is less than or equal to the output dimension, and the column vectors of $A$ are linearly independent, \textit{i.e.}, $\mathcal{H}^{(l)} = \{\mathbf{0}\}$, then the least harmful perturbation $(\delta^{(l)})^\ast$ is the eigenvector corresponding to the smallest eigenvalue of the matrix $A^\top A$.
\end{lemma}

\begin{proof} 
According to~\cref{remark2}, if the input dimension of the linear layer is less than or equal to the output dimension, and the column vectors of $A$ are linearly independent, then the dimension of the subspace for harmless perturbations is $dim(\mathcal{H}^{(l)}) = 0$. Thus, $\mathcal{H}^{(l)} = \{\mathbf{0}\}$, which means that there exists no (non-zero) harmless perturbations.

However, in this case, we can solve for the least harmful perturbation such that the output of the given linear layer is minimally affected. Given the matrix $A$ with equivalent effect of a linear layer,  the least harmful perturbation is 
\begin{equation} \label{eq:least_harmful_perturbation}
(\delta^{(l)})^\ast = {\rm arg min}_{\delta^{(l)}}\| A \delta^{(l)}\|_2, \textit{s.t.,} \|\delta^{(l)}\|_2=1.
\end{equation}
which can be rewritten as a quadratically constrained quadratic program,
\begin{equation} \label{eq:least_harmful_perturbation}
(\delta^{(l)})^\ast = {\rm arg min}_{\delta^{(l)}}\| A \delta^{(l)}\|^2_2, \textit{s.t.,} \|\delta^{(l)}\|^2_2=1.
\end{equation}
We solve above equation by introducing Lagrange function,
\begin{equation}
\begin{aligned}
 L(\delta^{(l)}, \gamma) &= \| A \delta^{(l)}\|^2_2 + \lambda(1-\|\delta^{(l)}\|^2_2) \\
&= (\delta^{(l)})^{\top} A^\top A \delta^{(l)} + \lambda (1-(\delta^{(l)})^{\top} \delta^{(l)})
 \end{aligned}
\end{equation}
Then the critical points of the Lagrange function can be found by, 
\begin{equation} \label{eq:lagrange_function}
\frac{\partial L(\delta^{(l)}, \gamma)}{\partial \delta^{(l)}} = 2A^\top A \delta^{(l)} - 2\lambda \delta^{(l)}=\mathbf{0}
\end{equation}

Then,~\cref{eq:lagrange_function} can be further written as $(A^\top A - \lambda I)\delta^{(l)}=\mathbf{0}$. That is, each eigenvector $\delta^{(l)}$ of $A^\top A$ with corresponding eigenvalue $\lambda$ is a critical point. To obtain the smallest $\| A \delta^{(l)}\|^2_2$, the smallest function value $\| A \delta^{(l)}\|^2_2$ at a critical point should be chosen. The solution of~\cref{eq:least_harmful_perturbation} is,
\begin{equation}
\| A \delta^{(l)}\|^2_2 = (\delta^{(l)})^\top A^\top A \delta^{(l)} = (\delta^{(l)})^\top \lambda \delta^{(l)} =  \lambda \|\delta^{(l)}\|^2_2 = \lambda
\end{equation}

Thus, the least harmful perturbation $(\delta^{(l)})^\ast$ is the eigenvector corresponding to the smallest eigenvalue  of the matrix $A^\top A$.

\end{proof}

\begin{lemma}
\label{lemma:no_harmless_perturbations_for_linear_layers}
If there exists no harmless perturbations in the $(l+1)$-th linear layer, \textit{i.e.}, $\mathcal{H}^{(l)}= \{\mathbf{0}\}$,  then there exist no two different perturbations $\forall \delta^{(l)}\ne \hat{\delta}^{(l)} \in \mathbb{R}^{n^{(l)}}$that produce the the same output of the layer, \textit{i.e.}, $f^{(l+1)}(\delta^{(l)})\ne f^{(l+1)}(\hat{\delta}^{(l)})$.
\end{lemma}
\begin{proof} 
Let the matrix $A \in \mathbb{R}^{n^{(l+1)}\times n^{(l)}}$ with linearly independent rows/columns have an equivalent effect with the parameters of the linear layer, \textit{i.e.}, $z^{(l+1)}=f^{(l+1)}(z^{(l)})=Az^{(l)}\in \mathbb{R}^{n^{(l+1)}}$. 

Let $\mathcal{H}^{(l)}=\{\delta^{(l)}\in \mathbb{R}^{n^{(l)}}|A(z^{(l)}+\delta^{(l)})=Az^{(l)}\}=\{\delta^{(l)}|A\delta^{(l)}=\mathbf{0}\} = N(A) = \{\mathbf{0}\}$.

We will prove  $\forall \delta^{(l)}\ne \hat{\delta}^{(l)} \in \mathbb{R}^{n^{(l)}}, f^{(l+1)}(\delta^{(l)})\ne f^{(l+1)}(\hat{\delta}^{(l)})$ by contradiction. For any two different perturbations $\forall \delta^{(l)}\ne \hat{\delta}^{(l)} \in \mathbb{R}^{n^{(l)}}$, it has $\delta^{(l)} - \hat{\delta}^{(l)} \ne \mathbf{0}$.

Assume that $\exists \delta^{(l)}\ne \hat{\delta}^{(l)} \in \mathbb{R}^{n^{(l)}}, f^{(l+1)}(\delta^{(l)}) = f^{(l+1)}(\hat{\delta}^{(l)})$. Then, it can be derived that $A\delta^{(l)} = A\hat{\delta}^{(l)}$ and is equivalent to $A(\delta^{(l)} - \hat{\delta}^{(l)})=\mathbf{0}$, which is in contradiction to $N(A) = \{\mathbf{0}\}$ since  $\delta^{(l)} - \hat{\delta}^{(l)} \ne \mathbf{0}$.

Thus, if $\mathcal{H}^{(l)}= \{\mathbf{0}\}$,  then there exist no two different perturbations $\forall \delta^{(l)}\ne \hat{\delta}^{(l)} \in \mathbb{R}^{n^{(l)}}$that produce the the same output of the layer, \textit{i.e.}, $f^{(l+1)}(\delta^{(l)})\ne f^{(l+1)}(\hat{\delta}^{(l)})$.

\end{proof}

\begin{lemma}
\label{lemma:different_direction_perturbation}
Given two different perturbations $\forall \delta^{(l)}\ne \hat{\delta}^{(l)} \in \mathbb{R}^{n^{(l)}}$ and $\delta^{(l)}, \hat{\delta}^{(l)} \notin \mathcal{H}^{(l)}$, $\mathcal{H}^{(l)}\ne \{\mathbf{0}\}$, if their corresponding orthogonal components have different directions, \textit{i.e.}, $\delta_{\bot}^{(l)} \ne \alpha \cdot \hat{\delta}_{\bot}^{(l)}, \alpha \in \mathbb{R}$, then $f^{(l+1)}(\delta^{(l)}) \ne  f^{(l+1)}(\hat{\delta}^{(l)})$.
\end{lemma}
\begin{proof} 
The linear layer is  $z^{(l+1)}=f^{(l+1)}(z^{(l)})=Az^{(l)}\in \mathbb{R}^{n^{(l+1)}}$. 

According to~\cref{theorem3}, $\forall \delta^{(l)}\in \mathbb{R}^{n^{(l)}}$ and $\delta^{(l)}\notin \mathcal{H}^{(l)}, f^{(l+1)}(\delta^{(l)}) = f^{(l+1)}(\delta_{\bot}^{(l)})$.

Similarly, $\forall \hat{\delta}^{(l)} \ne \delta^{(l)} \in \mathbb{R}^{n^{(l)}}$ and $\hat{\delta}^{(l)}\notin \mathcal{H}^{(l)}, f^{(l+1)}(\hat{\delta}^{(l)}) = f^{(l+1)}(\hat{\delta}_{\bot}^{(l)})$.

We will prove that if $\delta_{\bot}^{(l)} \ne \alpha \cdot \hat{\delta}_{\bot}^{(l)}, \alpha \in \mathbb{R}$, then $f^{(l+1)}(\delta^{(l)}) \ne  f^{(l+1)}(\hat{\delta}^{(l)})$ by contradiction.

Assume that $\exists \delta_{\bot}^{(l)} \ne \alpha \cdot \hat{\delta}_{\bot}^{(l)}$ such that $f^{(l+1)}(\delta^{(l)}) =  f^{(l+1)}(\hat{\delta}^{(l)})$. In this way, if $f^{(l+1)}(\delta^{(l)}) =  f^{(l+1)}(\hat{\delta}^{(l)})$, it has $f^{(l+1)}(\delta_{\bot}^{(l)}) = f^{(l+1)}(\delta^{(l)}) =    f^{(l+1)}(\hat{\delta}^{(l)}) = f^{(l+1)}(\hat{\delta}_{\bot}^{(l)})$, therefore, $A\delta_{\bot}^{(l)} = A\delta^{(l)} = A\hat{\delta}^{(l)}= A\hat{\delta}_{\bot}^{(l)}$. Furthermore, it can be derived that $A(\delta_{\bot}^{(l)} - \hat{\delta}_{\bot}^{(l)}) = \mathbf{0}$ and $\delta_{\bot}^{(l)} - \hat{\delta}_{\bot}^{(l)} \in N(A)$.

However, according to~\cref{theorem3}, these two orthogonal components are orthogonal to $\mathcal{H}^{(l)}= N(A) \ne \{\mathbf{0}\}$, respectively, \textit{i.e.}, $\delta_{\bot}^{(l)} \bot N(A)$ and $\hat{\delta}_{\bot}^{(l)} \bot N(A)$. In other words, $\delta_{\bot}^{(l)} \in N^{\bot}(A)$ and $\hat{\delta}_{\bot}^{(l)} \in N^{\bot}(A)$. Linear algebra states that, the row space $C(A^\top)$ of a matrix $A$ is orthogonal to the nullspace $N(A)$ of the matrix $A$, \textit{i.e.}, $C(A^\top) = N^{\bot}(A)$. Then, it has $\delta_{\bot}^{(l)} \in C(A^\top)$ and $\hat{\delta}_{\bot}^{(l)} \in C(A^\top)$. Thus, $\delta_{\bot}^{(l)} - \hat{\delta}_{\bot}^{(l)} \in C(A^\top)$.

Since (1) $\delta_{\bot}^{(l)} - \hat{\delta}_{\bot}^{(l)} \in N(A)$ and (2) $\delta_{\bot}^{(l)} - \hat{\delta}_{\bot}^{(l)} \in C(A^\top) = N^{\bot}(A)$, we can derive that $\delta_{\bot}^{(l)} - \hat{\delta}_{\bot}^{(l)} = \mathbf{0}$, which is in contradiction to $\delta_{\bot}^{(l)} \ne \alpha \cdot \hat{\delta}_{\bot}^{(l)}, \alpha \in \mathbb{R}$.

Thus, we have proved that given two different perturbations $\forall \delta^{(l)}\ne \hat{\delta}^{(l)} \in \mathbb{R}^{n^{(l)}}$, if their corresponding orthogonal components have different directions, \textit{i.e.}, $\delta_{\bot}^{(l)} \ne \alpha \cdot \hat{\delta}_{\bot}^{(l)}$, then $f^{(l+1)}(\delta^{(l)}) \ne  f^{(l+1)}(\hat{\delta}^{(l)})$.

\end{proof}

\section{Proof in~\cref{sec:projection_onto_subspace}}
\label{appx:orthogonal_decomposition} 
In the fifth paragraph of~\cref{sec:projection_onto_subspace}, we state that among all the perturbations leading to identical layer outputs, there exists a unique perturbation characterized by the smallest $\ell_2$ norm, which is orthogonal to the harmless perturbation subspace. 

\begin{proof} 
\cref{lemma:different_direction_perturbation} has proved that given two different perturbations $\forall \delta^{(l)}\ne \hat{\delta}^{(l)} \in \mathbb{R}^{n^{(l)}}$, if they lead to the same layer output, \textit{i.e.}, $f^{(l+1)}(\delta^{(l)}) =  f^{(l+1)}(\hat{\delta}^{(l)})$, then they share the same orthogonal component $\delta_{\bot}^{(l)} = \hat{\delta}_{\bot}^{(l)}$.


Next, we will prove that among all the perturbations that lead to the same layer output, the orthogonal perturbation $\delta^{(l)}_{\bot}$ has the smallest $\ell_2$ norm. 

According to~\cref{theorem3}, $\forall \delta^{(l)} \notin \mathcal{H}^{(l)}$,  given any perturbation $\forall \delta^{(l)}\in \mathbb{R}^{n^{(l)}}$ and $\delta^{(l)}\notin \mathcal{H}^{(l)}$, it has a unique decomposition $\delta^{(l)} = \delta^{(l)}_{\parallel} + \delta^{(l)}_{\bot}$, where $\delta^{(l)}_{\bot}\notin \mathcal{H}^{(l)}$ denotes the  component that is orthogonal to the harmless perturbation subspace $\mathcal{H}^{(l)}$. Since $\delta^{(l)}_{\parallel} \in \mathcal{H}^{(l)}$, then $\delta^{(l)}_{\bot} \bot \delta^{(l)}_{\parallel}$. Thus, it has $\|\delta^{(l)}\|^2 = \|\delta^{(l)}_{\parallel}\|^2 + \|\delta^{(l)}_{\bot}\|^2$. 

Furthermore,~\cref{theorem3} states that $f^{(l+1)}(\delta^{(l)}_{\parallel})= \mathbf{0}$ and $f^{(l+1)}(\delta^{(l)}) = f^{(l+1)}(\delta^{(l)}_{\bot})$. That is, $\|\delta^{(l)}_{\parallel}\|^2$ does not affect the network output. Thus, $\|\delta^{(l)}\|^2_{\text{\rm min}} = \|\delta^{(l)}_{\parallel}\|^2_{\text{\rm min}} + \|\delta^{(l)}_{\bot}\|^2 = 0+\|\delta^{(l)}_{\bot}\|^2 = \|\delta^{(l)}_{\bot}\|^2$.

Thus, among all the perturbations that lead to the same layer output, the orthogonal perturbation $\delta^{(l)}_{\bot}$ has the smallest $\ell_2$ norm.

\end{proof}

\section{Generation of an equivalence matrix $A$ for a convolutional layer}
\label{appx:equivalent_matrix}
In this section, we represent the convolutional layer as an equivalent matrix $A$, which is a sparse matrix.

Formally, let the matrix $A\in \mathbb{R}^{(C_{\text{\rm out}}  H_{\text{\rm out}} W_{\text{\rm out}})\times  (C_{\text{\rm in}}  H_{\text{\rm in}}  W_{\text{\rm in}})}$ with linearly independent rows/columns have an equivalent effect with the parameters of the convolutional layer $z^{(l+1)}=f^{(l+1)}(z^{(l)})$, \textit{i.e.}, $\hat{z}^{(l+1)}=A\hat{z}^{(l)}$, $\hat{z}^{(l)}\in \mathbb{R}^{C_{\text{\rm in}}  H_{\text{\rm in}}  W_{\text{\rm in}}}$ is the vectorized feature $z^{(l)}$, and $\hat{z}^{(l+1)}\in \mathbb{R}^{C_{\text{\rm out}}  H_{\text{\rm out}}  W_{\text{\rm out}}}$ is the vectorized feature $z^{(l+1)}$. 

To compute the equivalent matrix $A$ of the convolutional layer, we divide the generation process into three steps. (1) Consider a convolution kernel acting on a local receptive field as one equation. (2) Consider a convolution kernel acting on the whole input as $ H_{\text{\rm out}}  W_{\text{\rm out}}$ equations. (3) Consider $C_{\text{\rm out}}$ convolution kernels acting on the whole input as $C_{\text{\rm out}}  H_{\text{\rm out}}  W_{\text{\rm out}}$ equations.

(1) Consider a convolution kernel acting on a local receptive field as one equation.

Let $K\in \mathbb{R}^{c\times h\times w}$ denote a convolutional kernel and $k = [k_1, k_2, \cdots, k_{chw}]^\top \in \mathbb{R}^{ch w}$ denote the vectorized kernel. Let $X\in \mathbb{R}^{c\times h\times w}$ denote an input region/patch covered by a neuron's receptive field, and $x = [x_1, x_2, \cdots, x_{chw}]^\top \in \mathbb{R}^{chw}$ denote the vectorized input patch. Then, the scalar output of the convolutional kernel acting on this input patch is $k^\top x \in \mathbb{R}$. 

To compute a harmless perturbation patch $x$, it is equivalent to compute an equation $k^\top x = 0$, which represents the input patch satisfying the equation has no effect on the output, after passing through this convolutional kernel. In this way, the equivalent matrix $A = k^\top$.

(2) Consider a convolution kernel acting on the whole input as $ H_{\text{\rm out}}  W_{\text{\rm out}}$ equations.

Let $K\in \mathbb{R}^{c\times h\times w}$ denote a convolutional kernel and $k = [k_1, k_2, \cdots, k_{chw}]^\top \in \mathbb{R}^{ch w}$ denote the vectorized kernel. Let $z^{(l)}\in \mathbb{R}^{C_{\text{\rm in}} \times H_{\text{\rm in}} \times W_{\text{\rm in}}}$ denote the whole input feature, and $x \coloneqq \hat{z}^{(l)} = [x_1, x_2, \cdots, x_{C_{\text{\rm in}}  H_{\text{\rm in}}  W_{\text{\rm in}}}]^\top \in \mathbb{R}^{C_{\text{\rm in}}  H_{\text{\rm in}}  W_{\text{\rm in}}}$ denote the vectorized feature $z^{(l)}$. Here, $C_{\text{\rm in}}=c$.

Given a convolutional layer with the stride $S$  and the zero padding $O$, the convolutional kernel $K$ acts on the whole input feature $z^{(l)}$ to yield $H_{\text{\rm out}}  W_{\text{\rm out}}$ scalar outputs. Here, $H_{\text{\rm out}} = \lfloor \frac{H_{\text{\rm in}} - h + 2O}{S} \rfloor + 1$ and $W_{\text{\rm out}} = \lfloor \frac{W_{\text{\rm in}} - w + 2O}{S} \rfloor + 1$. In other words, there are a total of $H_{\text{\rm out}}  W_{\text{\rm out}}$ input patches covered by the receptive fields of a total of $H_{\text{\rm out}}  W_{\text{\rm out}}$ neurons. 

Let the $i$-th input patch $P^{(i)}\in \mathbb{R}^{c\times h\times w}$ denote the region covered by the receptive field of the $i$-th neuron, which has a total of $chw$ units.  Within the input region, each unit of $P^{(i)}$ is an unknown variable $x_j, j\in\{1, 2, \cdots, C_{\text{\rm in}}  H_{\text{\rm in}}  W_{\text{\rm in}}\}$ of the input feature $x$, or 0 if the unit is a zero padding unit in the receptive field. Let $p^{(i)}\in \mathbb{R}^{chw}$ denote the vectorized input  patch $P^{(i)}$.

To compute the harmless perturbation $x$,  let the outputs of the $H_{\text{\rm out}}  W_{\text{\rm out}}$ neurons all be zero. That is, the scalar output of each neuron needs to satisfy $k^\top p^{(i)} = 0, \forall i \in \{1, 2, \cdots,  H_{\text{\rm out}}  W_{\text{\rm out}}\}$. Thus, there are a total of $H_{\text{\rm out}}  W_{\text{\rm out}}$ equations with $C_{\text{\rm in}}  H_{\text{\rm in}}  W_{\text{\rm in}}$ unknown variables. Rewriting the $H_{\text{\rm out}}  W_{\text{\rm out}}$ equations in matrix form yields an equivalent matrix $A \in \mathbb{R}^{(H_{\text{\rm out}} W_{\text{\rm out}})\times  (C_{\text{\rm in}}  H_{\text{\rm in}}  W_{\text{\rm in}})}$ that satisfies $Ax=\mathbf{0}$. Here, the equivalent matrix $A$ is a sparse matrix, since $chw \ll C_{\text{\rm in}}  H_{\text{\rm in}}  W_{\text{\rm in}}$ is usually satisfied.

It is worth noting that when the kernel size of the convolutional layer is greater than or equal to the stride, the convolutional kernel acts on all $C_{\text{\rm in}}  H_{\text{\rm in}}  W_{\text{\rm in}}$ input variables, which means that each unknown variable is constrained by at least one equation. In this case, the matrix $A \in \mathbb{R}^{(H_{\text{\rm out}} W_{\text{\rm out}})\times  (C_{\text{\rm in}}  H_{\text{\rm in}}  W_{\text{\rm in}})}$ has linearly independent rows/columns. When the kernel size of the convolutional layer is smaller than the stride, the convolutional kernel only acts on some of input variables, which means that there are some unknown variables that are not constrained by any of the equations (variables not constrained by equations can take any value). In this case, the column vectors of the matrix $A\in \mathbb{R}^{(H_{\text{\rm out}} W_{\text{\rm out}})\times  (C_{\text{\rm in}}  H_{\text{\rm in}}  W_{\text{\rm in}})}$ are linearly dependent.

(3) Consider $C_{\text{\rm out}}$ convolution kernels acting on the whole input as $C_{\text{\rm out}}  H_{\text{\rm out}}  W_{\text{\rm out}}$ equations.

A convolutional kernel acting on the whole input yields $H_{\text{\rm out}}  W_{\text{\rm out}}$ equations ($H_{\text{\rm out}}  W_{\text{\rm out}}$ linearly independent row vectors in the matrix $A \in \mathbb{R}^{(H_{\text{\rm out}} W_{\text{\rm out}})\times  (C_{\text{\rm in}}  H_{\text{\rm in}}  W_{\text{\rm in}})}$).


Given a convolutional layer with linearly independent vectorized kernels, a total of $C_{\text{\rm out}}$ convolutional kernels acting on the whole input can yield $C_{\text{\rm out}}  H_{\text{\rm out}}  W_{\text{\rm out}}$ equations ($C_{\text{\rm out}}  H_{\text{\rm out}}  W_{\text{\rm out}}$ linearly independent row vectors in the matrix $A \in \mathbb{R}^{( C_{\text{\rm out}} H_{\text{\rm out}} W_{\text{\rm out}})\times  (C_{\text{\rm in}}  H_{\text{\rm in}}  W_{\text{\rm in}})}$).


\section{Experimental details}

\subsection{Experiment details for verifying the dimension of harmless perturbation subspace in~\cref{Fig:corollaries_harmless_perturbation}} \label{appx:verifying_dimension}

To verify the dimension of harmless perturbation subspace in~\cref{subsec:subspace_harmless_perturbations_for_linear_layers}, we fixed the input dimension and increased the output dimension of the first linear layer of each DNN trained on different dataset. To increase the output dimension of the linear layer, we increased the number of convolutional kernels $C_{\text{\rm out}}$ for the convolutional layer, and the number of neurons $N_{\text{\rm out}}$ for the fully-connected layer, respectively. 

For convolutional layers, if the input dimension is greater than the output
dimension, then the dimension of the harmless perturbation subspace for a convolutional layer is $dim(\mathcal{H}^{(l)}) = C_{\text{\rm in}}  H_{\text{\rm in}}  W_{\text{\rm in}} - C_{\text{\rm out}}  H_{\text{\rm out}} W_{\text{\rm out}}$. We modified the feature size of the output $H_{\text{\rm out}}$ and $W_{\text{\rm out}}$ of the first convolutional layer by setting different strides $S$ and kernel sizes $K$ ($K>S$). Here, $H_{\text{\rm out}} = \lfloor \frac{H_{\text{\rm in}} - K + 2P}{S} \rfloor + 1$ and $W_{\text{\rm out}} = \lfloor \frac{W_{\text{\rm in}} - K + 2P}{S} \rfloor + 1$. 

Specifically, (1) when the stride $S=1$, the kernel size $K=3$, and the zero padding $P=1$, it has $H_{\text{\rm out}} = H_{\text{\rm in}}$ and $W_{\text{\rm out}} = W_{\text{\rm in}}$. In this case, for the VGG-16 (stride $S= 1$) on the CIFAR-10 dataset, when $C_{\text{\rm out}} = C_{\text{\rm in}} = 3$, $dim(\mathcal{H}^{(l)}) = 0$. 

(2) Similarly, when $S=2, K=3$, and $P=1$, it has $H_{\text{\rm out}} = 0.5H_{\text{\rm in}}$ and $W_{\text{\rm out}} = 0.5W_{\text{\rm in}}$. In this case, for the ResNet-18/50 and EfficientNet (stride $S= 2$), when $C_{\text{\rm out}} = 4 C_{\text{\rm in}} = 12$, $dim(\mathcal{H}^{(l)}) = 0$. 

(3) When $S=4, K=5$, and $P=1$, it has $H_{\text{\rm out}} = 0.25H_{\text{\rm in}}$ and $W_{\text{\rm out}} = 0.25W_{\text{\rm in}}$. In this case, for the ResNet-18 (stride $S= 4$), when $C_{\text{\rm out}} = 16 C_{\text{\rm in}} = 48$, $dim(\mathcal{H}^{(l)}) = 0$.

For fully-connected layers, if the input dimension is greater than the output dimension, then the dimension of the harmless perturbation subspace for a fully-connected layer is $dim(\mathcal{H}^{(l)}) = N_{\text{\rm in}} - N_{\text{\rm out}}$.  We increased the number of neurons $N_{\text{\rm out}}$ of the first fully-connected layer. In this case, for the MLP-5 on the MNIST dataset, when $N_{\text{\rm out}} = N_{\text{\rm in}}=28\times 28 =784$, $dim(\mathcal{H}^{(l)}) = 0$.  For the MLP-5 on the CIFAR-10/100 and SVHN datasets, when $N_{\text{\rm out}} = N_{\text{\rm in}}=3\times 32\times 32 =3072$, $dim(\mathcal{H}^{(l)}) = 0$.

\begin{figure}[h]
\centering
\includegraphics[width=0.8\linewidth]{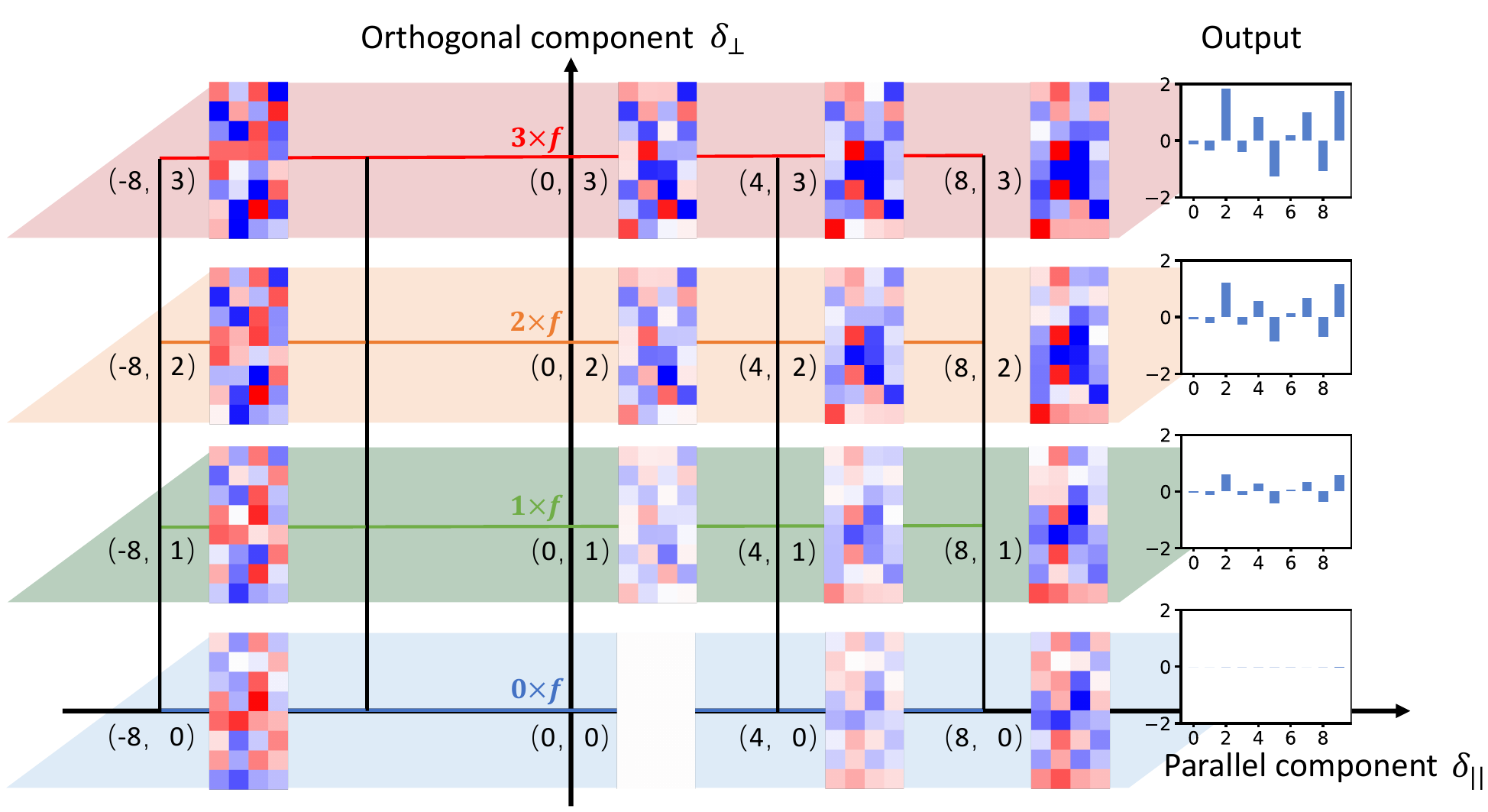}
\vskip -0.1in
\caption{Contour map of the actual impact
of any perturbations on the output of the last linear layer. Perturbations in the same row share identical orthogonal components, leading to identical network outputs. Perturbations in different rows exhibit proportional orthogonal components, leading to proportional network outputs.}
\label{Fig:example_of_theorems}
\vskip -0.2in
\end{figure}

\subsection{Experiment details for evaluating the network performance  in~\cref{Fig:accuracy_of_harmless_perturbation}} 
\label{appx:evaluating_performance}
To evaluate the impact of harmless perturbations  on network performance in~\cref{subsec:subspace_harmless_perturbations_for_linear_layers}, we trained the CIFAR-10
dataset on different networks and tested the effects of harmless perturbations on the network performance across varying perturbation magnitudes.

To ensure the existence of harmless perturbation subspace, it is necessary to guarantee that the input dimension of the convolutional layer exceeds the output dimension. Specifically, when computing harmless perturbations on images, we set the number of convolutional kernels to 10 in the first convolutional layer of the DNNs. When computing harmless perturbations on intermediate-layer features, we selected the first linear layer whose input dimension exceeded the output dimension.

When computing harmless perturbations on images, we replaced the original  convolutional layer in the ResNet18/50, which had 64  kernels with the kernel size 7, the stride 2, and the zero padding 3, with two convolutional layers. Specifically,  the first convolutional layer had 10 kernels with the kernel size 7, the stride 2, and the zero padding 3. The second convolutional layer had 64 kernels with the kernel size 3, the stride 1, and the zero padding 1. Thus, the dimension of the harmless perturbation subspace of the first convolutional layer after the replacement was $dim(\mathcal{H}^{(l)}) = C_{\text{\rm in}}  H_{\text{\rm in}}  W_{\text{\rm in}} - C_{\text{\rm out}}  H_{\text{\rm out}} W_{\text{\rm out}} = 3\times 32\times 32 - 10 \times 16 \times 16 = 512$.

When computing harmless perturbations on intermediate-layer features, we selected the first linear layer whose input dimension exceeds the output dimension. We removed the skip connections of selected convolutional layers. Specifically, for the ResNet-18, we computed the harmless perturbation subspace of the first convolutional layer of the 0-th block of the second layer. The dimension of the harmless perturbation subspace of the chosen convolutional layer was $dim(\mathcal{H}^{(l)}) = C_{\text{\rm in}}  H_{\text{\rm in}}  W_{\text{\rm in}} - C_{\text{\rm out}}  H_{\text{\rm out}} W_{\text{\rm out}} = 64\times 8\times 8 - 128 \times 4 \times 4 = 2048$. For the ResNet-50, we computed the harmless perturbation subspace of the first convolutional layer of the second block of the first layer. The dimension of the harmless perturbation subspace of the chosen convolutional layer was $dim(\mathcal{H}^{(l)}) = C_{\text{\rm in}}  H_{\text{\rm in}}  W_{\text{\rm in}} - C_{\text{\rm out}}  H_{\text{\rm out}} W_{\text{\rm out}} = 256\times 8\times 8 - 64 \times 8 \times 8 = 12288$.


Since there were infinite harmless perturbations in the harmless perturbation subspace, we chose a harmless perturbation direction to verify the network performance. Without loss of generality, we employed parallel components of adversarial perturbations (see~\cref{theorem3}) as the chosen directions of harmless perturbations. According to~\cref{theorem3}, the parallel component of an arbitrary perturbation is a harmless perturbation. 

To achieve this, we first generated adversarial perturbations $\delta^{\text{adv}}$ by PGD-20 for PGD with 20 steps, the maximum perturbation was set to $\epsilon = 8/255$ and the step size was set to $1/255$~\citep{madry2017towards}. Then, we produced the corresponding parallel components of the adversarial perturbations $\delta_{\parallel}^{\text{adv}} = P\delta^{\text{adv}}$ according to~\cref{theorem3}. Finally, we scaled the parallel components $\delta_{\parallel}^{\text{adv}}$ to obtain new perturbations, \textit{i.e.}, $\hat{\delta}_{\parallel}^{\text{adv}} = \frac{\epsilon}{\|\delta_{\parallel}^{\text{adv}}\|_\infty}  \cdot \delta_{\parallel}^{\text{adv}}$, such that the generated perturbations statisfied $\|\hat{\delta}_{\parallel}^{\text{adv}}\|_\infty = 8/255$. In~\cref{fig:corollaries_harmless_perturbation_conv}, we increased the magnitude of the generated perturbations by $\alpha \cdot \hat{\delta}_{\parallel}^{\text{adv}}$, where $\alpha = \{2^0, 2^1,\cdots, 2^8\}$.

\subsection{Experiment details for decomposing arbitrary perturbations  in~\cref{Fig:example_of_theorems}}
\label{appx:decomposition_perturbations}
To verify the decomposition of arbitrary perturbations in~\cref{theorem3} and~\cref{theorem4}, we plotted the contour map of the output of the last linear layer. According to~\cref{theorem4}, given a linear layer with a harmless perturbation subspace, the contour map of the layer output can be plotted along the direction of the chosen orthogonal component. We conducted experiments on MLP-5 with 32 neurons in the penultimate layer on the CIFAR-10 dataset.~\cref{Fig:example_of_theorems} illustrates the contour map of the actual impact of perturbations on the network output, in which perturbations were generated through linear combinations of orthogonal and parallel components of an arbitrary perturbation.


\subsection{Experiment details for privacy protection}
\label{appx:privacy_protection}
To achieve the application of the harmless
perturbation space for privacy protection in~\cref{privacy_protection}, we have to obtain the harmless perturbation subspace in the first linear layer. 

To ensure the existence of harmless perturbation subspace, we modified the number of convolution kernels to 10 in the first
convolutional layer in the ResNet-50. To achieve this, we replaced the original convolutional layer in the ResNet-50, which had 64  kernels with the kernel size 7, the stride 2, and the zero padding 3, with two convolutional layers. Specifically,  the first convolutional layer had 10 kernels with the kernel size 7, the stride 2, and the zero padding 3. The second convolutional layer had 64 kernels with the kernel size 3, the stride 1, and the zero padding 1. Thus, the dimension of the harmless perturbation subspace of the first convolutional layer after the replacement was $dim(\mathcal{H}^{(l)}) = C_{\text{\rm in}}  H_{\text{\rm in}}  W_{\text{\rm in}} - C_{\text{\rm out}}  H_{\text{\rm out}} W_{\text{\rm out}} = 3\times 32\times 32 - 10 \times 16 \times 16 = 512$.

\end{document}